\documentclass[11pt]{article}

\usepackage{fullpage,times,url,bm}

\usepackage{amsthm,amsfonts,amsmath,amssymb,epsfig,color,float,graphicx,verbatim}
\usepackage{algorithm,algorithmic}
\usepackage{bbm}
\usepackage{natbib}
\usepackage{enumitem}

\usepackage{graphicx}
\usepackage{subcaption}

\usepackage{array}

\usepackage{hyperref}
\hypersetup{
	colorlinks   = true, %Colours links instead of ugly boxes
	urlcolor     = blue, %Colour for external hyperlinks
	linkcolor    = blue, %Colour of internal links
	citecolor   = black %Colour of citations
}

\newtheorem{theorem}{Theorem}[section]

\newtheorem{lemma}{Lemma}[section]
\newtheorem{corollary}{Corollary}[section]
\newtheorem{definition}{Definition}[section]

\newcommand{\secref}[1]{Section~\ref{#1}}

\renewcommand{\eqref}[1]{Eq.~(\ref{#1})}
\newcommand{\lemref}[1]{Lemma~\ref{#1}}

\newcommand{\corollaryref}[1]{Corollary~\ref{#1}}
\newcommand{\thmref}[1]{Theorem~\ref{#1}}

\newcommand{\appref}[1]{Appendix~\ref{#1}}

\newcommand{\assref}[1]{Assumption~\ref{#1}}

\usepackage{amssymb}

\usepackage{bbm}
\newcommand{\onefunc}{\mathbbm{1}}

\newcommand{\stam}[1]{}
\newcommand{\ignore}[1]{}

\usepackage{mathtools}

\newtheorem{assumption}[theorem]{Assumption}

\newcommand{\bx}{\mathbf{x}}
\newcommand{\bw}{\mathbf{w}}

\newcommand{\bb}{\mathbf{b}}

\newcommand{\bz}{\mathbf{z}}

\newcommand{\by}{\mathbf{y}}

\newcommand{\bxi}{\boldsymbol{\xi}}

\newcommand{\btheta}{{\boldsymbol{\theta}}}

\newcommand{\bzeta}{\boldsymbol{\zeta}}

\newcommand{\ca}{{\cal A}}
\newcommand{\cb}{{\cal B}}
\newcommand{\cd}{{\cal D}}

\newcommand{\ce}{{\cal E}}

\newcommand{\cg}{{\cal G}}

\newcommand{\cl}{{\cal L}}
\newcommand{\cf}{{\cal F}}

\newcommand{\cz}{{\cal Z}}
\newcommand{\cs}{{\cal S}}
\newcommand{\cn}{{\cal N}}

\newcommand{\reals}{{\mathbb R}}

\newcommand{\zero}{{\mathbf{0}}}

\DeclareMathOperator{\poly}{poly}

\DeclareMathOperator*{\E}{\mathbb{E}}

\DeclareMathOperator{\xormaj}{XOR-MAJ}

\newcommand{\norm}[1]{\left\|#1\right\|}
\newcommand{\note}[1]{\textcolor{red}{\textbf{#1}}}

\makeatletter
\newcommand{\printfnsymbol}[1]{%
  \textsuperscript{\@fnsymbol{#1}}%
}
\makeatother

\title{Computational Complexity of Learning Neural Networks: Smoothness and Degeneracy}

\ignore{
\author{
        Amit Daniely\thanks{Hebrew University and Google, \texttt{amit.daniely@mail.huji.ac.il }}
	\and
	Nathan Srebro\thanks{TTI-Chicago, \texttt{nati@ttic.edu}}
	\and
	Gal Vardi\thanks{TTI-Chicago and Hebrew University, \texttt{galvardi@ttic.edu}}
}
}%ignore

\author{
{\small Amit Daniely}\\
{\small Hebrew University and Google}\\
{\small \texttt{amit.daniely@mail.huji.ac.il }}\\
\and
{\small Nathan Srebro}\\
{\small TTI-Chicago}\\
{\small \texttt{nati@ttic.edu}}\\
\and
{\small Gal Vardi}\\
{\small TTI-Chicago and Hebrew University}\\ 
{\small \texttt{galvardi@ttic.edu}}\\
\and
{\small Collaboration on the Theoretical Foundations of Deep Learning (\url{deepfoundations.ai})}
}

\date{}

\begin{document}

\maketitle

\begin{abstract}
	Understanding when neural networks can be learned efficiently	is a fundamental question in learning theory.
	Existing hardness results suggest that assumptions on both the input distribution and the network's weights are necessary for obtaining efficient algorithms. Moreover, it was previously shown that depth-$2$ networks can be efficiently learned under the assumptions that the input distribution is Gaussian, and the weight matrix is non-degenerate. In this work, we study whether such assumptions may suffice for learning deeper networks and prove negative results. We show that learning depth-$3$ ReLU networks under the Gaussian input distribution is hard even in the smoothed-analysis framework, where a random noise is added to the network's parameters. It implies that learning depth-$3$ ReLU networks under the Gaussian distribution is hard even if the weight matrices are non-degenerate.
	Moreover, we consider depth-$2$ networks, and show hardness of learning in the smoothed-analysis framework, where both the network parameters and the input distribution are smoothed.
	Our hardness results are under a well-studied assumption on the existence of local pseudorandom generators.
\end{abstract}

\section{Introduction}

The computational complexity of learning neural networks has been extensively studied in recent years, and there has been much effort in obtaining both upper bounds and hardness results. Nevertheless, it is still unclear when neural networks can be learned in polynomial time, namely, under what assumptions provably efficient algorithms exist. 

Existing results imply hardness already for learning depth-$2$ ReLU networks in the standard PAC learning framework (e.g., \cite{KlivansSh06,daniely2016complexity}). Thus, without any assumptions on the input distribution or the network's weights, efficient learning algorithms might not be achievable. Even when assuming that the input distribution is Gaussian, strong hardness results were obtained for depth-$3$ ReLU networks \citep{daniely2021local,chen2022hardness}, suggesting that assumptions merely on the input distribution might not suffice. Also, a hardness result by \cite{daniely2020hardness} shows that strong assumptions merely on the network's weights (without restricting the input distribution) might not suffice even for efficiently learning depth-$2$ networks. The aforementioned hardness results hold already for \emph{improper learning}, namely, where the learning algorithm is allowed to return a hypothesis that does not belong to the considered hypothesis class. Thus, a combination of assumptions on the input distribution and the network's weights seems to be necessary for obtaining efficient algorithms.

Several polynomial-time algorithms for learning depth-$2$ neural networks have been obtained, under assumptions on the input distribution and the network's weights \citep{awasthi2021efficient,janzamin2015beating,zhong2017recovery,ge2017earning,bakshi2019learning,ge2018learning}.
In these works, it is assumed that the weight matrices are non-degenerate. That is, they either assume that the condition number of the weight matrix is bounded, or some similar non-degeneracy assumption. Specifically, \cite{awasthi2021efficient} gave an efficient algorithm for learning depth-$2$ (one-hidden-layer) ReLU networks, that may include bias terms in the hidden neurons, under the assumption that the input distribution is Gaussian, and the weight matrix is non-degenerate. The non-degeneracy assumption holds w.h.p. if we add a small random noise to any weight matrix, and hence their result implies efficient learning of depth-$2$ ReLU networks under the Gaussian distribution in the \emph{smoothed-analysis} framework. 

The positive results on depth-$2$ networks suggest the following question: 
\begin{quote}
	\emph{Is there an efficient algorithm for learning ReLU networks of depth larger than $2$ under the Gaussian distribution, where the weight matrices are non-degenerate, or in the smoothed-analysis framework where the network's parameters are smoothed?}
\end{quote}
In this work, we give a negative answer to this question, already for depth-$3$ networks\footnote{We note that in our results the neural networks have the ReLU activation also in the output neuron.}. We show that learning depth-$3$ ReLU networks under the Gaussian distribution is hard even in the smoothed-analysis framework, where a random noise is added to the network's parameters. 
As a corollary, we show that learning depth-$3$ ReLU networks under the Gaussian distribution is hard even if the weight matrices are non-degenerate.
Our hardness results are under a well-studied cryptographic assumption on the existence of \emph{local pseudorandom generators (PRG)} with polynomial stretch. 

Motivated by the existing positive results on smoothed-analysis in depth-$2$ networks, we also study whether learning depth-$2$ networks with smoothed parameters can be done under weak assumptions on the input distribution. Specifically, we consider the following question:
\begin{quote}
	\emph{Is there an efficient algorithm for learning depth-$2$ ReLU networks in the smoothed-analysis framework, where both the network's parameters and the input distribution are smoothed?}
\end{quote}
We give a negative answer to this question, by showing hardness of learning depth-$2$ ReLU networks where a random noise is added to the network's parameters, and the input distribution on $\reals^d$ is obtained by smoothing an i.i.d. Bernoulli distribution on $\{0,1\}^d$. This hardness result is also under the assumption on the existence of local PRGs.

\subsection*{Related work}

\paragraph{Hardness of learning neural networks.}

Hardness of improperly learning depth-$2$ neural networks 
%(with respect to the square loss) 
follows from hardness of improperly learning DNFs or intersections of halfspaces, since these classes can be expressed by depth-$2$ networks.
\cite{KlivansSh06} showed, assuming the hardness of the shortest vector problem, that learning intersections of 
%$n^\epsilon$ 
halfspaces 
%for a constant $\epsilon>0$, 
is hard. 
Hardness of learning DNF formulas is implied by \cite{applebaum2010public} under a combination of two assumptions: the first is related to the planted dense subgraph problem in hypergraphs, and the second is related to local PRGs.
\cite{daniely2016complexity} showed hardness of learning DNFs
%with $q(n)=\omega(\log(n))$ terms, 
under 
%the assumption
a common assumption, namely,
that refuting a random $K$-SAT formula is hard.
%Hence, the results from \cite{KlivansSh06} and \cite{daniely2016complexity} imply hardness of improperly learning depth-$2$ neural networks.
% with $n^\epsilon$ and $\omega(\log(n))$ hidden neurons (respectively).
All of the above results are distribution-free, namely, they do not imply hardness of learning neural networks under some specific distribution. 

%In  \cite{Kharitonov93}, it is shown that learning depth-$d$ circuits under the uniform distribution on the hypercube is hard, where $d$ is an unspecified sufficiently large constant, under a relatively strong assumption on the complexity of factoring random Blum integers. 
 \cite{applebaum2016fast} showed, under an assumption on a specific candidate for Goldreich's PRG (i.e., based on a predicate called $\xormaj$), that learning depth-$3$ Boolean circuits under the uniform distribution on the hypercube is hard.
 \cite{daniely2021local} proved distribution-specific hardness of learning Boolean circuits of depth-$2$ (namely, DNFs) and depth-$3$, under the assumpion on the existence of local PRGs that we also use in this work. For DNF formulas, they showed hardness of learning under a distribution where each component is drawn i.i.d. from a Bernoulli distribution (which is not uniform). For depth-$3$ Boolean circuits, they showed hardness of learning under the uniform distribution on the hypercube. 
Since the Boolean circuits can be expressed by ReLU networks of the same depth, these results readily translate to distribution-specific hardness of learning neural networks. 
\cite{chen2022hardness} showed hardness of learning depth-$2$ neural networks under the uniform distribution on the hypercube, based on an assumption on the hardness of the \emph{Learning with Rounding (LWR)} problem.
Note that the input distributions in the above results are supported on the hypercube, and they do not immediately imply hardness of learning neural networks under continuous distributions.

When considering the computational complexity of learning neural networks, perhaps the most natural choice of an input distribution is the standard Gaussian distribution. \cite{daniely2021local} established hardness of learning depth-$3$ networks under this distribution, based on the assumpion on the existence of local PRGs. \cite{chen2022hardness} also showed hardness of learning depth-$3$ networks under the Gaussian distribution, but their result holds already for networks that do not have an activation function in the output neuron, and it is based on the LWR assumption. They also showed hardness of learning constant depth ReLU networks from label queries (i.e., where the learner has the ability to query the value of the target network at any desired input) under the Gaussian distribution, based either on the decisional Diffie-Hellman or the Learning with Errors assumptions.
 
The above results suggest that assumptions on the input distribution might not suffice for achieving an efficient algorithm for learning depth-$3$ neural networks. A natural question is whether assumptions on the network weights may suffice. \cite{daniely2020hardness} showed (under the assumption that refuting a random $K$-SAT formula is hard) that distribution-free learning of depth-$2$ neural networks is hard already if the weights are drawn from some ``natural'' distribution or satisfy some ``natural'' properties. Thus, if we do not impose any assumptions on the input distribution, then even very strong assumptions on the network's weights do not suffice for efficient learning.

Several works in recent years have shown hardness of distribution-specific learning shallow neural networks using gradient-methods or statistical query (SQ) algorithms \citep{shamir2018distribution,song2017complexity,vempala2019gradient,goel2020superpolynomial,diakonikolas2020algorithms,chen2022hardness}. 
Specifically, \cite{goel2020superpolynomial} and \cite{diakonikolas2020algorithms} gave superpolynomial \emph{correlational SQ} (CSQ) lower bounds for learning depth-$2$ networks under the Gaussian distribution. Since there is separation between CSQ and SQ, these results do not imply hardness of general SQ learning.  \cite{chen2022hardness} showed general SQ lower bounds for depth-$3$ networks under the Gaussian distribution.
It is worth noting that while the SQ framework captures some variants of the gradient-descent algorithm, it does not capture, for example, stochastic gradient-descent (SGD), which examines training points individually (see a discussion in \cite{goel2020superpolynomial}).

We emphasize that none of the above distribution-specific hardness results for neural networks (either for improper learning or SQ learning) holds in the smoothed analysis framework or for non-degenerate weights.  

While hardness of proper learning is implied by hardness of improper learning, there are also works that show hardness of properly learning depth-$2$ networks under more standard assumptions (cf. \cite{goel2020tight}). Finally, distribution-specific hardness of learning a single ReLU neuron in the agnostic setting was shown in \cite{goel2019time,goel2020statistical,diakonikolas2020near}. 

\paragraph{Learning neural networks in polynomial time.}

\cite{awasthi2021efficient} gave a polynomial-time algorithm for learning depth-$2$ (one-hidden-layer) ReLU networks,
%(without activation in the output neuron), 
under the assumption that the input distribution is Gaussian, and the weight matrix of the target network is non-degenerate. Their algorithm is based on tensor decomposition, and it can handle bias terms in the hidden layer. Their result also implies that depth-$2$ ReLU networks with Gaussian inputs can be learned efficiently under the smoothed-analysis framework. Our work implies that such a result might not be possible in depth-$3$ networks (with activation in the output neuron). Prior to \cite{awasthi2021efficient}, several works gave polynomial time algorithms for learning depth-$2$ neural networks where the input distribution is Gaussian and the weight matrix is non-degenerate \citep{janzamin2015beating,zhong2017recovery,ge2017earning,ge2018learning,bakshi2019learning}, but these works either do not handle the presence of bias terms 
%\citep{ge2017earning,bakshi2019learning,ge2018learning}, 
or do not handle the ReLU activation. 
%\citep{janzamin2015beating}. 
Some of the aforementioned works consider networks with multiple outputs, and allow certain non-Gaussian input distributions. 

Provable guarantees for learning neural networks in super-polynomial time were given in \cite{diakonikolas2020algorithms,chen2022learning,diakonikolas2020small,goel2019learning,vempala2019gradient,zhang2016l1,goel2017reliably}. Distribution-free learning of a single ReLU neuron (in the realizable setting) can be done in polynomial time \citep{kalai2009isotron,kakade2011efficient}.

\section{Preliminaries}

%\paragraph{Notations.}
\subsection{Notation}

We use bold-face letters to denote vectors, e.g., $\bx=(x_1,\ldots,x_d)$. 
For a vector $\bx$ and a sequence $S=(i_1,\ldots,i_k)$ of $k$ indices, we let $\bx_S=(x_{i_1},\ldots,x_{i_k})$, i.e., the restriction of $\bx$ to the indices $S$.
We denote by $\onefunc[\cdot]$ the indicator function, for example $\onefunc[t \geq 5]$ equals $1$ if $t \geq 5$ and $0$ otherwise. 
%We denote $\sign(z) = \onefunc[z \geq 0]$.
For an integer $d \geq 1$ we denote $[d]=\{1,\ldots,d\}$.
%For a set $A$ we denote by $\cu(A)$ the uniform distribution over $A$. 
We denote by $\cn(0,\sigma^2)$ the normal distribution with mean $0$ and variance $\sigma^2$, and by $\cn(\zero,\Sigma)$ the multivariate normal distribution with mean $\zero$ and covariance matrix $\Sigma$. The identity matrix of size $d$ is denoted by $I_d$.
For $\bx \in \reals^d$ we denote by $\norm{\bx}$ the Euclidean norm.
We use $\poly(a_1,\ldots,a_n)$ to denote a polynomial in $a_1,\ldots,a_n$.
%We use standard asymptotic notations $\co(\cdot)$, $\Omega(\cdot)$ to hide constant factors, and $\tilde{\co}(\cdot)$, $\tilde{\Omega}(\cdot)$ to hide logarithmic factors.

\subsection{Local pseudorandom generators}

An $(n,m,k)$-hypergraph is a hypergraph over $n$ vertices $[n]$ with $m$ hyperedges $S_1,\ldots,S_m$, each of cardinality $k$. Each hyperedge $S = (i_1,\ldots,i_k)$ is ordered, and all the $k$ members of a hyperedge are distinct.
We let $\cg_{n,m,k}$ be the distribution over such hypergraphs in which a hypergraph is chosen by picking each hyperedge uniformly and independently at random among all the possible $n \cdot (n-1) \cdot \ldots \cdot (n-k+1)$ ordered hyperedges.
Let $P:\{0,1\}^k \rightarrow \{0,1\}$ be a predicate, and let $G$ be a $(n,m,k)$-hypergraph.
We call {\em Goldreich's pseudorandom generator (PRG)} \citep{goldreich2000candidate} the function $f_{P,G}:\{0,1\}^n \rightarrow \{0,1\}^m$ such that for $\bx \in \{0,1\}^n$, we have $f_{P,G}(\bx) = (P(\bx_{S_1}),\ldots,P(\bx_{S_m}))$.
The integer $k$ is called the {\em locality} of the PRG. If $k$ is a constant then the PRG and the predicate $P$ are called {\em local}.
We say that the PRG has {\em polynomial stretch} if $m = n^s$ for some constant $s>1$.
We let $\cf_{P,n,m}$ denote the collection of functions $f_{P,G}$ where $G$ is an $(n,m,k)$-hypergraph. We sample a function from $\cf_{P,n,m}$ by choosing a random hypergraph $G$ from $\cg_{n,m,k}$.
%short cancelled

We denote by $G \xleftarrow{R} \cg_{n,m,k}$ the operation of sampling a hypergraph $G$ from $\cg_{n,m,k}$, and by $\bx \xleftarrow{R} \{0,1\}^n$ the operation of sampling $\bx$ from the uniform distribution on $\{0,1\}^n$.
We say that $\cf_{P,n,m}$ is $\varepsilon$-pseudorandom generator ($\varepsilon$-PRG) if for every polynomial-time probabilistic algorithm $\ca$ the {\em distinguishing advantage}
\[
\left|
\Pr_{G \xleftarrow{R} \cg_{n,m,k}, \bx \xleftarrow{R} \{0,1\}^n}[\ca(G,f_{P,G}(\bx))=1] -
\Pr_{G \xleftarrow{R} \cg_{n,m,k}, \by \xleftarrow{R} \{0,1\}^m}[\ca(G,\by)=1]
\right|
\]
is at most $\varepsilon$. Thus, the distinguisher $\ca$ is given a random hypergraph $G$ and a string $\by \in \{0,1\}^m$, and its goal is to distinguish between the case where $\by$ is chosen at random, and the case where $\by$ is a random image of $f_{P,G}$.

Our assumption is 
%short cancelled
%We assume
that local PRGs with polynomial stretch and constant distinguishing advantage exist:
\begin{assumption}
\label{ass:localPRG}
For every constant $s>1$, there exists a constant $k$ and a predicate $P:\{0,1\}^k \rightarrow \{0,1\}$, such that $\cf_{P,n,n^s}$ is $\frac{1}{3}$-PRG.
\end{assumption}

%short cancelled
We remark that the same assumption was used by \citet{daniely2021local} to show hardness-of-learning results. 
%short
%Note that we assume constant distinguishing advantage. In the literature, a requirement of negligible distinguishing advantage\footnote{More formally, that for $1-o_n(1)$ fraction of the hypergraphs, the distinguisher has no more than negligible advantage.} is often considered (cf. \citep{applebaum2016algebraic,applebaum2016cryptographic,couteau2018concrete}). Thus, our requirement from the PRG is weaker.
%short cancelled
Local PRGs have been extensively studied in the last two decades.
%short cancelled
In particular, local PRGs with polynomial stretch have shown to have remarkable applications, such as secure-computation with constant computational overhead \citep{ishai2008cryptography,applebaum2017secure}, and general-purpose obfuscation based on constant degree multilinear maps (cf. \citet{lin2016indistinguishability,linVai2016indistinguishability}).
Significant evidence for Assumption~\ref{ass:localPRG} was shown in \citet{applebaum2013pseudorandom}. 
%and \citet{applebaum2016algebraic}. 
Moreover, a concrete candidate for a local PRG, based on the XOR-MAJ predicate, was shown to be secure against all known attacks \citep{applebaum2016algebraic,couteau2018concrete,meaux2019improved,applebaum2016fast}. 
See \citet{daniely2021local} for further discussion on the 
%assumption.
%short cancelled
assumption, and on prior work regarding the relation between Goldreich's PRG and hardness of learning.

%ignore

\subsection{Neural networks}

We consider feedforward ReLU networks. Starting from an input $\bx \in \reals^d$, each layer in the network is of the form $\bz \mapsto \sigma(W_i \bz + \bb_i)$, where $\sigma(a) = [a]_+ = \max\{0,a\}$ is the ReLU activation which applies to vectors entry-wise, $W_i$ is the weight matrix, and $\bb_i$ are the bias terms. 
The {\em weights vector} of the $j$-th neuron in the $i$-th layer is the $j$-th row of $W_i$, and its {\em outgoing-weights vector} is the $j$-th column of $W_{i+1}$.
We define the \emph{depth} of the network as the number of layers.
Unless stated otherwise, the output neuron also has a ReLU activation function. Note that a depth-$k$ network with activation in the output neuron has $k$ non-linear layers.
A neuron which is not an input or output neuron is called a {\em hidden neuron}.
We sometimes consider neural networks with multiple outputs.
The parameters of the neural network is the set of its weight matrices and bias vectors. We often view the parameters as a vector $\btheta \in \reals^p$ obtained by concatenating these matrices and vectors.
For $B \geq 0$, we say that the parameters are $B$-bounded if the absolute values of all weights and biases are at most $B$.

%ignore

\subsection{Learning neural networks and the smoothed-analysis framework}

We first define learning neural networks under the standard PAC framework:

\begin{definition}[Distribution-specific PAC learning]  \label{def:PAC}
	Learning depth-$k$ neural networks under an input distribution $\cd$ on $\reals^d$ 
	%(in the standard PAC model) 
	is defined by the following framework:
\begin{enumerate}
	\item An adversary chooses a set of $B$-bounded parameters $\btheta \in \reals^p$ for a depth-$k$ neural network $N_\btheta:\reals^d \to \reals$, as well as some $\epsilon > 0$. 
	\item Consider an examples oracle, such that each example $(\bx,y) \in \reals^d \times \reals$ is drawn i.i.d. with $\bx \sim \cd$
	and $y = N_{\btheta}(\bx)$. 
	Then, given access to the examples oracle, the goal of the learning algorithm $\cl$ is to return with probability at least $\frac{3}{4}$ a hypothesis $h: \reals^d \to \reals$ such that 
\[
	%\E_{\bx \sim \cn(\zero, I_d)} \left[ \left(h(\bx) - N_{\hat{\btheta}}(\bx) \right)^2 \right] \leq \frac{1}{10}~.
	\E_{\bx \sim \cd} \left[ \left(h(\bx) - N_{\btheta}(\bx) \right)^2 \right] \leq \epsilon~.
\]
	We say that $\cl$ is \emph{efficient} if it runs in time $\poly(d,p,B,1/\epsilon)$.
\end{enumerate}
\end{definition}

We consider learning in the \emph{smoothed-analysis} framework \citep{spielman2004smoothed}, which is a popular paradigm for analyzing non-worst-case computational complexity \citep{roughgarden2021beyond}. 
The smoothed-analysis 
framework has been successfully applied to many learning problems (e.g., \cite{awasthi2021efficient,ge2018learning,kalai2009learning,bhaskara2014smoothed,bhaskara2019smoothed,haghtalab2020smoothed,ge2015learning,kalai2008decision,brutzkus2020id3}).
In the smoothed-analysis setting, the target network is not purely controlled by an adversary. Instead, the adversary can first generate an arbitrary network, and the parameters for this network (i.e., the weight matrices and bias terms) will be randomly perturbed to yield a perturbed network. The algorithm only needs to work with high probability on the perturbed network. This limits the power of the adversary and prevents it from creating highly degenerate cases. 
Formally,  we consider the following framework (we note that a similar model was considered in \cite{awasthi2021efficient} and \cite{ge2018learning}):
\begin{definition}[Learning with smoothed parameters]  \label{def:smoothed params}
	Learning depth-$k$ neural networks with smoothed parameters under an input distribution $\cd$ is defined as follows: 
        \begin{enumerate}
            	\item An adversary chooses a set of $B$-bounded parameters $\btheta \in \reals^p$ for a depth-$k$ neural network $N_\btheta:\reals^d \to \reals$, as well as some $\tau, \epsilon > 0$. 
            	\item A perturbed set of parameters $\hat{\btheta}$ is obtained by a random perturbation to $\btheta$, namely, $\hat{\btheta} = \btheta + \bxi$ for $\bxi \sim \cn(\zero, \tau^2 I_p)$.
            	\item Consider an examples oracle, such that each example $(\bx,y) \in \reals^d \times \reals$ is drawn i.i.d. with $\bx \sim \cd$ and $y = N_{\hat{\btheta}}(\bx)$. Then, given access to the examples oracle, the goal of the learning algorithm $\cl$ is to return with probability at least $\frac{3}{4}$ (over the random perturbation $\bxi$ and the internal randomness of $\cl$) a hypothesis $h: \reals^d \to \reals$ such that 
            \[
            	%\E_{\bx \sim \cn(\zero, I_d)} \left[ \left(h(\bx) - N_{\hat{\btheta}}(\bx) \right)^2 \right] \leq \frac{1}{10}~.
            	\E_{\bx \sim \cd} \left[ \left(h(\bx) - N_{\hat{\btheta}}(\bx) \right)^2 \right] \leq \epsilon~.
            \]
            	We say that $\cl$ is \emph{efficient} if it runs in time $\poly(d,p,B,1/\epsilon,1/\tau)$.
        \end{enumerate}
\end{definition}
	
	Finally, we also consider a setting where both the parameters and the input distribution are smoothed: 
\begin{definition}[Learning with smoothed  parameters and inputs] \label{def:smoothed params and inputs}
	Learning depth-$k$ neural networks with smoothed parameters and inputs under an input distribution $\cd$ is defined as follows:
        \begin{enumerate}
        	\item An adversary chooses a set of $B$-bounded parameters $\btheta \in \reals^p$ for a depth-$k$ neural network $N_\btheta:\reals^d \to \reals$, as well as some $\tau, \omega, \epsilon > 0$. 
        	\item A perturbed set of parameters $\hat{\btheta}$ is obtained by a random perturbation to $\btheta$, namely, $\hat{\btheta} = \btheta + \bxi$ for $\bxi \sim \cn(\zero, \tau^2 I_p)$. Moreover, a smoothed input distribution $\hat{\cd}$ is obtained 
        	%by convolving the density function of $\cd$ with a Gaussian of covariance $\omega^2 I_d$.
        	from $\cd$, such that $\hat{\bx} \sim \hat{\cd}$ is chosen by drawing $\bx \sim \cd$ and adding a random perturbation from $\cn(\zero, \omega^2 I_d)$.
        	\item Consider an examples oracle, such that each example $(\bx,y) \in \reals^d \times \reals$ is drawn i.i.d. with $\bx \sim \hat{\cd}$	and $y = N_{\hat{\btheta}}(\bx)$. Then, given access to the examples oracle, the goal of the learning algorithm $\cl$ is to return with probability at least $\frac{3}{4}$ (over the random perturbation $\bxi$ and the internal randomness of $\cl$) a hypothesis $h: \reals^d \to \reals$ such that 
        \[
        	%\E_{\bx \sim \cn(\zero, I_d)} \left[ \left(h(\bx) - N_{\hat{\btheta}}(\bx) \right)^2 \right] \leq \frac{1}{10}~.
        	\E_{\bx \sim \hat{\cd}} \left[ \left(h(\bx) - N_{\hat{\btheta}}(\bx) \right)^2 \right] \leq \epsilon~.
        \]
        	We say that $\cl$ is \emph{efficient} if it runs in time $\poly(d,p,B,1/\epsilon,1/\tau,1/\omega)$.
        \end{enumerate}
\end{definition}

We emphasize that all of the above definitions consider learning in the distribution-specific setting. Thus, the learning algorithm may depend on the specific input distribution $\cd$. 

\section{Results}

As we discussed in the introduction, there exist efficient algorithms for learning depth-$2$ (one-hidden-layer) ReLU networks with smoothed parameters under the Gaussian distribution. We now show that such a result may not be achieved for depth-$3$ networks:

\begin{theorem} \label{thm:hard smoothed}
 	Under \assref{ass:localPRG}, there is no efficient algorithm that learns 
	%in the smoothed-analysis framework 
	depth-$3$ neural networks with smoothed parameters (in the sense of Definition~\ref{def:smoothed params}) under the standard Gaussian distribution.
\end{theorem}

%See \appref{app:proof of hard smoothed} for the proof.
We prove the theorem in \secref{sec:proof of hard smoothed}.
Next, we conclude that learning depth-$3$ neural networks 
%from $\reals^d$ to $\reals$ 
under the Gaussian distribution on $\reals^d$ is hard in the standard PAC framework even if all weight matrices are non-degenerate, namely, when the minimal singular values of the weight matrices are lower bounded by $1/\poly(d)$. 
As we discussed in the introduction, in one-hidden-layer networks 
%under the Gaussian input distribution and such a non-degeneracy assumption, 
with similar assumptions
there exist efficient learning algorithms. 

\begin{corollary} \label{cor:non-degenerate}
	Under \assref{ass:localPRG}, there is no efficient algorithm that learns depth-$3$ neural networks (in the sense of Definition~\ref{def:PAC}) under the standard Gaussian distribution on $\reals^d$, even if the smallest singular value of each weight matrix is at least $1/\poly(d)$.
\end{corollary}

The proof of the above corollary follows from \thmref{thm:hard smoothed}, using the fact that by adding 
%a random perturbation of size $1/\poly(d)$ to the weight matrices, we get w.h.p. matrices where the smallest singular values are lower bounded by some $1/\poly(d)$. 
a small random perturbation to the weight matrices, we get w.h.p. non-degenerate matrices.
Hence, an efficient algorithm that learns non-degenerate networks suffices for obtaining an efficient algorithm that learns under the smoothed analysis framework. See \appref{app:proof of cor} for the formal proof.

The above hardness results consider depth-$3$ networks that include activation in the output neuron. Thus, the networks have three non-linear layers. We remark that these results readily imply hardness also for depth-$4$ networks without activation in the output, and hardness for depth-$k$ networks for any $k>3$. We also note that the hardness results hold already for networks where all hidden layers are of the same width, e.g., where all layers are of width $d$.

%The above hardness results give 
\thmref{thm:hard smoothed} gives a
strong limitation on learning depth-$3$ networks in the smoothed-analysis framework. Motivated by existing positive results on smoothed-analysis in depth-$2$ networks, we now study
%An additional natural question is 
whether learning depth-$2$ networks with smoothed parameters can be done under weak assumptions on the input distribution. 
%Hence, we now 
Specifically, we
consider the case where both the parameters and the input distribution are smoothed.
We show that efficiently learning depth-$2$ networks may not be possible with smoothed parameters where the input distribution on $\reals^d$ is obtained by smoothing an i.i.d. Bernoulli distribution on $\{0,1\}^d$. 

\begin{theorem} \label{thm:smoothed weights and inputs}
	Under \assref{ass:localPRG}, there is no efficient algorithm that learns depth-$2$ neural networks with smoothed parameters and inputs (in the sense of Definition~\ref{def:smoothed params and inputs}), under the distribution $\cd$ on $\{0,1\}^d$ where each coordinate is drawn i.i.d. from a Bernoulli distribution which takes the value $0$ with probability $\frac{1}{\sqrt{d}}$.	
\end{theorem}

We prove the theorem in  \appref{app:proof of smoothed weights and inputs}.
The proof follows from similar ideas to the proof of \thmref{thm:hard smoothed}, which we discuss in the next section.
%The proof of \thmref{thm:smoothed weights and inputs} follows from similar ideas to the proof of \thmref{thm:hard smoothed}, which we discuss in the next section.
%See \appref{app:proof of smoothed weights and inputs} for the formal proof.

\section{Proof of \thmref{thm:hard smoothed}} \label{sec:proof of hard smoothed}

The proof builds on a technique from \cite{daniely2021local}. It follows by showing that an efficient algorithm for learning depth-$3$ neural networks with smoothed parameters under the Gaussian distribution can be used for breaking a local PRG. Intuitively, the main challenge in our proof in comparison to \cite{daniely2021local} is that our reduction must handle the random noise which is added to the parameters. Specifically, \cite{daniely2021local} define a certain examples oracle, and show that the examples returned by the oracle are realizable by some neural network which depends on the unknown $\bx \in \{0,1\}^n$ used by the PRG. Since the network depends on this unknown $\bx$, some of the parameters of the network are unknown, and hence it is not trivial how to define an examples oracle which is realizable by a perturbed network. Moreover, we need to handle this random perturbation without increasing the network's depth.  

We now provide the formal proof. The proof relies on several lemmas which we prove in \appref{app:missing lemmas}.
For a sufficiently large $n$, let $\cd$ be the standard Gaussian distribution on $\reals^{n^2}$. Assume that there is a $\poly(n)$-time algorithm $\cl$ that learns depth-$3$ neural networks 
with at most $n^2$ hidden neurons
and parameter magnitudes
bounded by $n^3$, 
with smoothed parameters,
under the distribution $\cd$, 
%in the smoothed-analysis framework 
with $\epsilon=\frac{1}{n}$, and $\tau=1/\poly(n)$ that we will specify later. Let $m(n) \leq \poly(n)$ be the sample complexity of $\cl$, namely, $\cl$ uses a sample of size at most $m(n)$ and returns with probability at least $\frac34$ a hypothesis $h$ with $\E_{\bz \sim \cd} \left[ \left(h(\bz) - N_{\hat{\btheta}}(\bz) \right)^2 \right] \leq \epsilon = \frac{1}{n}$, where $N_{\hat{\btheta}}$ is the perturbed network. Let $s>1$ be a constant such that $n^s \geq m(n) + n^3$ for every sufficiently large $n$. By \assref{ass:localPRG}, there exists a constant $k$ and a predicate $P:\{0,1\}^k \to \{0,1\}$, such that $\cf_{P,n,n^s}$ is $\frac{1}{3}$-PRG. We will show an efficient algorithm $\ca$ with distinguishing advantage greater than $\frac13$ and thus reach a contradiction.

Throughout this proof, we will use the following notations.
For a hyperedge $S = (i_1,\ldots,i_k)$ we denote by $\bz^S \in \{0,1\}^{kn}$ the following encoding of $S$: the vector $\bz^S$ is a concatenation of $k$ vectors in $\{0,1\}^n$, such that the $j$-th vector has $0$ in the $i_j$-th coordinate and $1$ elsewhere. Thus, $\bz^S$ consists of $k$ size-$n$ slices, each encoding a member of $S$. For $\bz \in \{0,1\}^{kn}$, $i \in [k]$ and $j \in [n]$, we denote $z_{i,j} = z_{(i-1)n + j}$. That is, $z_{i,j}$ is the $j$-th component in the $i$-th slice in $\bz$. For $\bx \in \{0,1\}^n$, let $P_\bx: \{0,1\}^{kn} \to \{0,1\}$ be such that for every hyperedge $S$ we have $P_\bx(\bz^S) = P(\bx_S)$. Let $c$ be such that $\Pr_{t \sim \cn(0,1)}[t \leq c] = \frac{1}{n}$. Let $\mu$ be the density of $\cn(0,1)$, let $\mu_-(t) = n \cdot \onefunc[t \leq c] \cdot \mu(t)$, and let $\mu_+ = \frac{n}{n-1} \cdot \onefunc[t \geq c] \cdot \mu(t)$. Note that $\mu_-,\mu_+$ are the densities of the restriction of $\mu$ to the intervals $t \leq c$ and $t \geq c$ respectively. 
Let $\Psi: \reals^{kn} \to \{0,1\}^{kn}$ be a mapping such that for every $\bz' \in \reals^{kn}$ and $i \in [kn]$ we have $\Psi(\bz')_i = \onefunc[z'_i \geq c]$. For $\tilde{\bz} \in \reals^{n^2}$ we denote $\tilde{\bz}_{[kn]} = (\tilde{z}_1,\ldots,\tilde{z}_{kn})$, namely, the first $kn$ components of $\tilde{\bz}$ (assuming $n^2 \geq kn$).

\subsection{Defining the target network for $\cl$}

Since our goal is to use the algorithm $\cl$ for breaking PRGs, in this subsection we define a neural network $\tilde{N}:\reals^{n^2} \to \reals$ that we will later use as a target network for $\cl$. The network $\tilde{N}$ contains the subnetworks $N_1,N_2,N_3$ which we define below.

Let $N_1$ be a depth-$2$ neural network with input dimension $kn$, at most $n \log(n)$ hidden neurons, at most $\log(n)$ output neurons (with activations in the output neurons), and parameter magnitudes bounded by $n^3$ (all bounds are for a sufficiently large $n$), which satisfies the following. We denote the set of output neurons of $N_1$ by $\ce_1$.
Let $\bz' \in \reals^{kn}$ be an input to $N_1$ such that $\Psi(\bz') = \bz^S$ for some hyperedge $S$, and assume that for every $i \in [kn]$ we have $z'_i \not \in \left( c, c + \frac{1}{n^2} \right)$. Fix some $\bx \in \{0,1\}^n$. Then, for $S$ with $P_\bx(\bz^S) = 0$ the inputs to all output neurons $\ce_1$ are at most $-1$, and for $S$ with $P_\bx(\bz^S) = 1$ there exists a neuron in $\ce_1$ with input at least $2$. 
Recall that our definition of a neuron's input includes the addition of the bias term.
The construction of the network $N_1$ is given in \lemref{lem:network N1}. 
Intuitively, the network $N_1$ consists of a layer that transforms w.h.p. the input $\bz'$ to $\Psi(\bz')=\bz^S$, followed by a layer that satisfies the following: Building on a lemma from \cite{daniely2021local} which shows that $P_\bx(\bz^S)$ can be computed by a DNF formula, we define a layer where each output neuron corresponds to a term in the DNF formula, such that if the term evaluates to 0 then the input to the neuron is at most $-1$, and otherwise it is at least $2$.
Note that the network $N_1$ depends on $\bx$. However, 
%as we show in \lemref{lem:network N1}, 
only the second layer depends on $\bx$, and thus given an input we may compute the first layer even without knowing $\bx$.
Let $N'_1: \reals^{kn} \to \reals$ be a depth-$3$ neural network with no activation function in the output neuron, obtained from $N_1$ by summing the outputs from all neurons $\ce_1$.

Let $N_2$ be a depth-$2$ neural network with input dimension $kn$, at most $n \log(n)$ hidden neurons, at most $2n$ output neurons, and parameter magnitudes bounded by $n^3$ (for a sufficiently large $n$), which satisfies the following. We denote the set of output neurons of $N_2$ by $\ce_2$.
Let $\bz' \in \reals^{kn}$ be an input to $N_2$ such that for every $i \in [kn]$ we have $z'_i \not \in \left(c, c + \frac{1}{n^2} \right)$. If $\Psi(\bz')$ is an encoding of a hyperedge then the inputs to all output neurons $\ce_2$ are at most $-1$, and otherwise there exists a neuron in $\ce_2$ with input at least $2$.
The construction of the network $N_2$ is given in \lemref{lem:network N2}. 
Intuitively, each neuron in $\ce_2$ is responsible for checking whether $\Psi(\bz')$ violates some requirement that must hold in an encoding of a hyperedge. 
%Note that the network $N_2$ is independent of $\bx$.
Let $N'_2: \reals^{kn} \to \reals$ be a depth-$3$ neural network with no activation function in the output neuron, obtained from $N_2$ by summing the outputs from all neurons $\ce_2$.

Let $N_3$ be a depth-$2$ neural network with input dimension $kn$, at most $n \log(n)$ hidden neurons, $kn \leq n \log(n)$ output neurons, and parameter magnitudes bounded by $n^3$ (for a sufficiently large $n$), which satisfies the following. We denote the set of output neurons of $N_3$ by $\ce_3$. 
Let $\bz' \in \reals^{kn}$ be an input to $N_3$. If there exists $i \in [kn]$ such that $z'_i \in \left(c, c + \frac{1}{n^2} \right)$ then there exists a neuron in $\ce_3$ with input at least $2$. Moreover, if for all $i \in [kn]$ we have $z'_i \not \in \left(c - \frac{1}{n^2}, c + \frac{2}{n^2} \right)$ then the inputs to all neurons in $\ce_3$ are at most $-1$.
The construction of the network $N_3$ is straightforward and given in \lemref{lem:network N3}.  
%Note that the network $N_3$ is independent of $\bx$. 
Let $N'_3: \reals^{kn} \to \reals$ be a depth-$3$ neural network with no activation function in the output neuron, obtained from $N_3$ by summing the outputs from all neurons $\ce_3$.

Let $N': \reals^{kn} \to \reals$ be a depth-$3$ network obtained from $N'_1,N'_2,N'_3$ as follows. For $\bz' \in \reals^{kn}$ we have $N'(\bz') = \left[ 1 - N'_1(\bz') - N'_2(\bz') - N'_3(\bz')  \right]_+$. The network $N'$ has at most $n^2$ neurons, and parameter magnitudes bounded by $n^3$ (all bounds are for a sufficiently large $n$).
Finally, let $\tilde{N}:\reals^{n^2} \rightarrow \reals$ be a depth-$3$ neural network such that $\tilde{N}(\tilde{\bz}) = N'\left(\tilde{\bz}_{[kn]}\right)$.

\subsection{Defining the noise magnitude $\tau$ and analyzing the perturbed network}

%In this proof, we show how to use the efficient algorithm $\cl$ in order to break PRGs. Recall that the algorithm $\cl$ learns neural networks under the smoothed-analysis framework. Hence, in 
In order to use the algorithm $\cl$ w.r.t. some neural network with parameters $\btheta$, we need to implement an examples oracle, such that the examples are labeled according to a neural network with parameters $\btheta+\bxi$, where $\bxi$ is a random perturbation.
Specifically, we use $\cl$ with an examples oracle where the labels correspond to a network $\hat{N}:\reals^{n^2} \to \reals$, obtained from $\tilde{N}$ (w.r.t. an appropriate $\bx \in \{0,1\}^n$ in the construction of $N_1$) by adding a small perturbation to the parameters. The perturbation is such that we add i.i.d. noise to each parameter in $\tilde{N}$, where the noise is distributed according to $\cn(0,\tau^2)$, and $\tau=1/\poly(n)$ is small enough such that the following holds. 
Let $f_\btheta:\reals^{n^2} \to \reals$ be any depth-$3$ neural network parameterized by $\btheta \in \reals^r$ for some $r>0$ with at most $n^2$ neurons, and parameter magnitudes bounded by $n^3$ (note that $r$ is polynomial in $n$). 
%and let $\tilde{\bz}  \in \reals^{n^2}$  with $\norm{\tilde{\bz}} \leq 2n$. 
Then with probability at least $1-\frac{1}{n}$ over $\bxi \sim \cn(\zero, \tau^2 I_r)$, we have $| \xi_i | \leq \frac{1}{10}$ for all $i \in [r]$, and the network $f_{\btheta + \bxi}$ is such that for every input $\tilde{\bz}  \in \reals^{n^2}$ with $\norm{\tilde{\bz}} \leq 2n$ and every neuron we have: Let $a,b$ be the inputs to the neuron in the computations $f_\btheta(\tilde{\bz})$ and $f_{\btheta + \bxi}(\tilde{\bz})$ (respectively), then $|a-b| \leq \frac12$.  Thus, $\tau$ is sufficiently small, such that w.h.p. adding i.i.d. noise $\cn(0,\tau^2)$ to each parameter does not change the inputs to the neurons by more than $\frac12$. Note that such an inverse-polynomial $\tau$ exists, since when the network size, parameter magnitudes, and input size are bounded by some $\poly(n)$, then the input to each neuron in $f_\btheta(\tilde{\bz})$ is $\poly(n)$-Lipschitz as a function of $\btheta$, and thus it suffices to choose $\tau$ that implies with probability at least $1 - \frac{1}{n}$ that $\norm{\bxi} \leq \frac{1}{q(n)}$ for a sufficiently large polynomial $q(n)$ (see \lemref{lem:tau exists} for details). 

Let $\tilde{\btheta} \in \reals^p$ be the parameters of the network $\tilde{N}$. Recall that the parameters vector $\tilde{\btheta}$ is the concatenation of all weight matrices and bias terms. Let $\hat{\btheta} \in \reals^p$ be the parameters of $\hat{N}$, namely, $\hat{\btheta} = \tilde{\btheta} + \bxi$ where $\bxi \sim \cn(\zero, \tau^2 I_p)$. 
%Let $\bz' = \tilde{\bz}_{[kn]}$. 
By our choice of $\tau$ and the construction of the networks $N_1,N_2,N_3$, with probability at least $1-\frac{1}{n}$ over $\bxi$, for every $\tilde{\bz}$ with $\norm{\tilde{\bz}} \leq 2n$, the inputs to the neurons $\ce_1,\ce_2,\ce_3$ in the computation $\hat{N}(\tilde{\bz})$ satisfy the following properties, where we denote  $\bz' = \tilde{\bz}_{[kn]}$: 
\begin{enumerate}[label=(P\arabic*)]
	\item If $\Psi(\bz') = \bz^S$ for some hyperedge $S$, and for every $i \in [kn]$ we have $z'_i \not \in \left( c, c + \frac{1}{n^2} \right)$, then the inputs to $\ce_1$ satisfy: \label{prop:first}
	\begin{itemize}
		\item If $P_\bx(\bz^S) = 0$ the inputs to all neurons in $\ce_1$ are at most $-\frac12$. 
		\item If $P_\bx(\bz^S) = 1$ there exists a neuron in $\ce_1$ with input at least $\frac32$.
	\end{itemize}
	\item  If for every $i \in [kn]$ we have $z'_i \not \in \left(c, c + \frac{1}{n^2} \right)$, then the inputs to $\ce_2$ satisfy:  \label{prop:second}
	\begin{itemize}
		\item If $\Psi(\bz')$ is an encoding of a hyperedge then the inputs to all neurons $\ce_2$ are at most $-\frac12$.
		\item Otherwise, there exists a neuron in $\ce_2$ with input at least $\frac32$.
	\end{itemize}
	\item The inputs to $\ce_3$ satisfy: \label{prop:last}
	\begin{itemize}
		\item  If there exists $i \in [kn]$ such that $z'_i \in \left(c, c + \frac{1}{n^2} \right)$ then there exists a neuron in $\ce_3$ with input at least $\frac32$. 
		\item If for all $i \in [kn]$ we have $z'_i \not \in \left(c - \frac{1}{n^2}, c + \frac{2}{n^2} \right)$ then the inputs to all neurons in $\ce_3$ are at most $-\frac12$.
	\end{itemize}
\end{enumerate}

\subsection{Stating the algorithm $\ca$}

Given a sequence $(S_1,y_1),\ldots,(S_{n^s},y_{n^s})$, where $S_1,\ldots,S_{n^s}$ are i.i.d. random hyperedges, the algorithm $\ca$ needs to distinguish whether $\by = (y_1,\ldots,y_{n^s})$ is random or that $\by = (P(\bx_{S_1}),\ldots,P(\bx_{S_{n^s}})) = (P_\bx(\bz^{S_1}),\ldots,P_\bx(\bz^{S_{n^s}}))$ for a random $\bx \in \{0,1\}^n$.
Let $\cs = ((\bz^{S_1},y_1),\ldots,(\bz^{S_{n^s}},y_{n^s}))$.

We use the efficient algorithm $\cl$ in order to obtain distinguishing advantage greater than $\frac{1}{3}$ as follows.
Let $\bxi$ be a random perturbation, and let $\hat{N}$ be the perturbed network as defined above, w.r.t. the unknown $\bx \in \{0,1\}^n$. Note that given a perturbation $\bxi$, only the weights in the second layer of the subnetwork $N_1$ in $\hat{N}$ are unknown, since all other parameters do not depend on $\bx$.
%, and we know the random perturbation $\bxi$.
The algorithm $\ca$ runs $\cl$ with the following examples oracle.
In the $i$-th call, the oracle first draws $\bz \in \{0,1\}^{kn}$ such that each component is drawn i.i.d. from a Bernoulli distribution which takes the value $0$ with probability $\frac{1}{n}$. If $\bz$ is an encoding of a hyperedge then the oracle replaces $\bz$ with $\bz^{S_i}$. Then, the oracle chooses $\bz' \in \reals^{kn}$ such that for each component $j$, if $z_j = 1$ then $z'_j$ is drawn from $\mu_+$, and otherwise $z'_j$ is drawn from $\mu_-$.
Let $\tilde{\bz} \in \reals^{n^2}$ be such that $\tilde{\bz}_{[kn]}=\bz'$, and the other $n^2-kn$ components of $\tilde{\bz}$ are drawn i.i.d. from $\cn(0,1)$.
Note that the vector $\tilde{\bz}$ has the distribution $\cd$, due to the definitions of the densities $\mu_+$ and $\mu_-$, and since replacing an encoding of a random hyperedge by an encoding of another random hyperedge does not change the distribution of $\bz$.
Let $\hat{b} \in \reals$ be the bias term of the output neuron of $\hat{N}$.
The oracle returns $(\tilde{\bz},\tilde{y})$, where the labels $\tilde{y}$ are chosen as follows:
\begin{itemize}
	\item If $\Psi(\bz')$ is not an encoding of a hyperedge, then $\tilde{y}=0$.
	\item If $\Psi(\bz')$ is an encoding of a hyperedge:
    	\begin{itemize}
    		\item If $\bz'$ does not have components in the interval $(c-\frac{1}{n^2},c+\frac{2}{n^2})$, then if $y_i=0$ we set $\tilde{y} = \hat{b}$, and if $y_i=1$ we set $\tilde{y} = 0$.
			%$\tilde{y} = \hat{b} - y_i$.
    		\item If $\bz'$ has a component in the interval $(c,c+\frac{1}{n^2})$, then $\tilde{y} = 0$.
    		\item If $\bz'$ does not have components in the interval $(c,c+\frac{1}{n^2})$, but has a component in the interval $(c-\frac{1}{n^2},c+\frac{2}{n^2})$, 
		%then $\tilde{y} = \hat{N}(\tilde{\bz})$. %$\tilde{y} = [1 - y_i - N_3(\bz')]_+$.
		then the label $\tilde{y}$ is determined as follows: 
		\begin{itemize}
			\item If $y_i=1$ then $\tilde{y} = 0$.
			\item If $y_i=0$: Let $\hat{N}_3$ be the network $\hat{N}$ after omitting the neurons $\ce_1,\ce_2$ and their incoming and outgoing weights. Then, we set $\tilde{y} = [\hat{b} - \hat{N}_3(\tilde{\bz})]_+$. Note that since only the second layer of $N_1$ depends on $\bx$, then we can compute $\hat{N}_3(\tilde{\bz})$ without knowing $\bx$.
		\end{itemize}
    	\end{itemize}
\end{itemize}

Let $h$ be the hypothesis returned by $\cl$.
Recall that $\cl$ uses at most $m(n)$ examples, and hence $\cs$ contains at least $n^3$ examples that $\cl$ cannot view. We denote the indices of these examples by $I = \{m(n)+1,\ldots,m(n)+n^3\}$, and the examples by $\cs_I = \{(\bz^{S_i},y_i)\}_{i \in I}$. By $n^3$ additional calls to the oracle, the algorithm $\ca$ obtains the examples $\tilde{\cs}_I = \{(\tilde{\bz}_i,\tilde{y}_i)\}_{i \in I}$ that correspond to $\cs_I$.
Let $h'$ be a hypothesis such that for all $\tilde{\bz} \in \reals^{n^2}$ we have $h'(\tilde{\bz}) = \max\{0,\min\{\hat{b},h(\tilde{\bz})\}\}$, thus, for $\hat{b} \geq 0$ the hypothesis $h'$ is obtained from $h$ by clipping the output to the interval $[0,\hat{b}]$.
Let $\ell_{I}(h')=\frac{1}{|I|}\sum_{i \in I}(h'(\tilde{\bz}_i)-\tilde{y}_i)^2$. 
Now, if $\ell_I(h') \leq \frac{2}{n}$, then $\ca$ returns $1$, and otherwise it returns $0$.
We remark that the decision of our algorithm is based on $h'$ (rather than $h$) since we need the outputs to be bounded, in order to allow using Hoeffding's inequality in our analysis, which we discuss in the next subsection.

\subsection{Analyzing the algorithm $\ca$}

Note that the algorithm $\ca$ runs in $\poly(n)$ time.
We now show that if $\cs$ is pseudorandom then $\ca$ returns $1$ with probability greater than $\frac{2}{3}$, and if $\cs$ is random then $\ca$ returns $1$ with probability less than $\frac{1}{3}$.

We start with the case where $\cs$ is pseudorandom. 
In \lemref{lem:realizable}, we prove that if $\cs$ is pseudorandom then w.h.p. (over $\bxi \sim \cn(\zero, \tau^2 I_p)$ and the i.i.d. inputs $\tilde{\bz}_i \sim \cd$) the examples $(\tilde{\bz}_1, \tilde{y}_1),\ldots,(\tilde{\bz}_{m(n)+n^3},\tilde{y}_{m(n)+n^3})$ returned by the oracle are realized by $\hat{N}$. Thus, $\tilde{y}_i = \hat{N}(\tilde{\bz}_i)$ for all $i$.
As we show in the lemma, this claim follows by noting that the following hold w.h.p., where we denote $\bz'_i = (\tilde{\bz}_i)_{[kn]}$: 
\begin{itemize}
	\item If $\Psi( \bz'_i )$ is not an encoding of a hyperedge, then the oracle sets $\tilde{y}_i=0$, and we have:
	\begin{itemize}
		\item If $\bz'_i$ does not have components in $\left( c, c+\frac{1}{n} \right)$, then there exists a neuron in $\ce_2$ with output at least $\frac{3}{2}$ (by Property~\ref{prop:second}), which implies $\hat{N}(\tilde{\bz}_i)=0$.
				\item If $\bz'$ has a component in $\left( c, c+\frac{1}{n} \right)$, then there exists a neuron in $\ce_3$ with output at least $\frac{3}{2}$ (by Property~\ref{prop:last}), which implies $\hat{N}(\tilde{\bz}_i)=0$.
	\end{itemize} 
	 \item If $\Psi(\bz'_i)$ is an encoding of a hyperedge $S$, then by the definition of the examples oracle we have $S=S_i$. Hence:
		\begin{itemize}
			
			\item If $\bz'_i$ does not have components in $\left( c - \frac{1}{n^2}, c + \frac{2}{n^2} \right)$, then:
			
			\begin{itemize}
				
				\item	If $y_i = 0$ then the oracle sets $\tilde{y}_i=\hat{b}$. Since $\cs$ is pseudorandom, we have $P_\bx(\bz^S) = P_\bx(\bz^{S_i}) = y_i = 0$. Hence, in the computation $\hat{N}(\tilde{\bz}_i)$ the inputs to all neurons in $\ce_1,\ce_2,\ce_3$ are at most $-\frac{1}{2}$ (by Properties~\ref{prop:first},~\ref{prop:second} and~\ref{prop:last}), and thus their outputs are $0$. Therefore, $\hat{N}(\tilde{\bz}_i) = \hat{b}$.
				
				\item If $y_i = 1$ then the oracle sets $\tilde{y}_i=0$. Since $\cs$ is pseudorandom, we have $P_\bx(\bz^S) = P_\bx(\bz^{S_i}) = y_i = 1$. Hence, in the computation $\hat{N}(\tilde{\bz}_i)$ there exists a neuron in $\ce_1$ with output at least $\frac{3}{2}$ (by Property~\ref{prop:first}), which implies $\hat{N}(\tilde{\bz}_i)=0$. 
			
			\end{itemize} 
			
			\item If $\bz'_i$ has a component in $\left( c , c + \frac{1}{n^2} \right)$, then the oracle sets $\tilde{y}_i=0$. Also, in the computation $\hat{N}(\tilde{\bz}_i)$ there exists a neuron in $\ce_3$ with output at least $\frac{3}{2}$ (by Property~\ref{prop:last}), which implies $\hat{N}(\tilde{\bz}_i)=0$. 			
			\item  If $\bz'_i$ does not have components in the interval $(c,c+\frac{1}{n^2})$, but has a component in the interval $(c-\frac{1}{n^2},c+\frac{2}{n^2})$, then:
			
			\begin{itemize}
				
				\item If $y_i = 1$ the oracle sets $\tilde{y}_i=0$. Since $\cs$ is pseudorandom, we have $P_\bx(\bz^S) = P_\bx(\bz^{S_i}) = y_i = 1$. Hence, in the computation $\hat{N}(\tilde{\bz}_i)$ there exists a neuron in $\ce_1$ with output at least $\frac{3}{2}$ (by Property~\ref{prop:first}), which implies $\hat{N}(\tilde{\bz}_i)=0$.  
				
				\item  If $y_i = 0$ the oracle sets $\tilde{y}_i=[\hat{b} - \hat{N}_3(\tilde{\bz}_i)]_+$. Since $\cs$ is pseudorandom, we have $P_\bx(\bz^S) = P_\bx(\bz^{S_i}) = y_i = 0$. Therefore, in the computation $\hat{N}(\tilde{\bz}_i)$ all neurons in $\ce_1,\ce_2$ have output $0$ (by Properties~\ref{prop:first} and~\ref{prop:second}), and  hence their contribution to the output of $\hat{N}$ is $0$. Thus, by the definition of $\hat{N}_3$, we have $\hat{N}(\tilde{\bz}_i) = [\hat{b} - \hat{N}_3(\tilde{\bz}_i)]_+$.			
			\end{itemize}			
		\end{itemize}
\end{itemize}
Recall that the algorithm $\cl$ is such that with probability at least $\frac{3}{4}$ (over $\bxi \sim \cn(\zero, \tau^2 I_p)$, the i.i.d. inputs $\tilde{\bz}_i \sim \cd$, and its internal randomness), given a size-$m(n)$ dataset labeled by $\hat{N}$, it returns a hypothesis $h$ such that $\E_{\tilde{\bz} \sim \cd} \left[(h(\tilde{\bz})-\hat{N}(\tilde{\bz}))^2 \right] \leq \frac{1}{n}$.
By the definition of $h'$ and the construction of $\hat{N}$, if $h$ has small error then $h'$ also has small error, namely, we have $ \E_{\tilde{\bz} \sim \cd} \left[(h'(\tilde{\bz})-\hat{N}(\tilde{\bz}))^2 \right]  \leq \frac{1}{n}$.
\ignore{
\[
	 \E_{\tilde{\bz} \sim \cd} \left[(h'(\tilde{\bz})-\hat{N}(\tilde{\bz}))^2 \right] 
	 \leq \E_{\tilde{\bz} \sim \cd} \left[(h(\tilde{\bz})-\hat{N}(\tilde{\bz}))^2 \right] 
	 \leq \frac{1}{n}~.
\]
}%ignore
In \lemref{lem:pseudorandom small loss} we use the above arguments and Hoeffding's inequality over $\tilde{\cs}_I$, and prove that with probability greater than $\frac{2}{3}$ we have $\ell_I(h') \leq \frac{2}{n}$.

%ignore

Next, we consider the case where $\cs$ is random.
Let $\tilde{\cz} \subseteq \reals^{n^2}$ be such that $\tilde{\bz} \in \tilde{\cz}$ iff $\tilde{\bz}_{[kn]}$ does not have components in the interval $(c-\frac{1}{n^2},c+\frac{2}{n^2})$, and $\Psi(\tilde{\bz}_{[kn]})=\bz^{S}$ for a hyperedge $S$.
If $\cs$ is random, then by the definition of our examples oracle, for every $i \in [m(n)+n^3]$ such that $\tilde{\bz}_i \in \tilde{\cz}$, we have $\tilde{y}_i=\hat{b}$ with probability $\frac{1}{2}$ and $\tilde{y}_i=0$ otherwise. Also, by the definition of the oracle, $\tilde{y}_i$ is independent of $S_i$ and independent of the choice of the vector $\tilde{\bz}_i$ that corresponds to $\bz^{S_i}$. Hence, for such $\tilde{\bz}_i \in \tilde{\cz}$ with $i \in I$, any hypothesis cannot predict the label $\tilde{y}_i$, and the expected loss for the example is at least $\left(\frac{\hat{b}}{2}\right)^2$.
Moreover, in \lemref{lem:prob z good} we show that $\Pr\left[\tilde{\bz}_i \in \tilde{\cz}\right] \geq \frac{1}{2\log(n)}$ for a sufficiently large $n$.
In \lemref{lem:random large loss} we use these arguments to prove a lower bound on $\E_{\tilde{\cs}_I}\left[ \ell_I(h') \right]$, and by Hoeffding's inequality over $\tilde{\cs}_I$ we conclude that with probability greater than $\frac23$ we have $\ell_I(h') > \frac{2}{n}$.

%ignore

Overall, if $\cs$ is pseudorandom then with probability greater than $\frac{2}{3}$ the algorithm $\ca$ returns $1$, and if $\cs$ is random then with probability greater than $\frac{2}{3}$ the algorithm $\ca$ returns $0$. Thus, the distinguishing advantage is greater than $\frac13$.

\section{Discussion}

Understanding the computational complexity of learning neural networks is a central question in learning theory. Our results imply that the assumptions which allow for efficient learning in one-hidden-layer networks might not suffice in deeper networks. Also, in depth-$2$ networks we show that it is not sufficient to assume that both the parameters and the inputs are smoothed. 
We hope that our hardness results will help focus on assumptions that may allow for efficient learning.
Below we discuss several intriguing open problems.

First, we emphasize that our hardness results are for neural networks that include the ReLU activation also in the output neuron. In contrast, the positive results on learning depth-$2$ networks that we discussed in the introduction do not include activation in the output neuron. Therefore, as far as we are aware, there is still a gap between the upper bounds and our hardness results: (1) Under the assumption that the input is Gaussian and the weights are non-degenerate, the cases of depth-$2$ networks with activation in the output neuron and of depth-$3$ networks without activation in the output are not settled; (2) In the setting where both the parameters and the input distribution are smoothed, the case of depth-$2$ networks without activation in the output is not settled. 

Moreover, our hardness result for depth-$3$ networks suggests that adding a small (yet polynomial) noise to the parameters is not sufficient to allow efficient learning under the Gaussian distribution. An intriguing direction is to explore the case where the noise is larger. Intuitively, the case of very large noise corresponds to learning a random network, where the weights are drawn i.i.d. from the Gaussian distribution. Hence, it would be interesting to understand whether there is an efficient algorithm for learning random ReLU networks under the Gaussian input distribution. Likewise, the effect of large noise may also be considered w.r.t. the inputs, namely, we may use large noise for obtaining smoothed input distributions.

\subsection*{Acknowledgements}

This work was done as part of the NSF-Simons Sponsored Collaboration on the Theoretical Foundations of Deep Learning.
%AD, NS, and GV acknowledge the support of the NSF and the Simons Foundation for the Collaboration on the Theoretical Foundations of Deep Learning through awards DMS-2031883 and \#814639.

\bibliographystyle{abbrvnat}
\bibliography{bib}

\newpage

\appendix

%{\hypersetup{linkcolor=black}\tableofcontents} 

\section{Missing lemmas for the proof of \thmref{thm:hard smoothed}} \label{app:missing lemmas}

%The following lemma is from \cite{daniely2021local}, and is required for the construction of $N_1$. For completeness, we give here both the lemma and its proof. 

\begin{lemma}[\cite{daniely2021local}] \label{lem:from P to DNF}
	For every predicate $P:\{0,1\}^k \rightarrow \{0,1\}$ and $\bx \in \{0,1\}^n$, there is a DNF formula $\psi$ over $\{0,1\}^{kn}$ with at most $2^k$ terms, such that for every hyperedge $S$ we have $P_\bx(\bz^S)=\psi(\bz^S)$. Moreover, each term in $\psi$ is a conjunction of positive literals.
\end{lemma}
\begin{proof}
	The following proof is from \cite{daniely2021local}, and we give it here for completeness.	

	We denote by $\cb \subseteq \{0,1\}^{k}$ the set of satisfying assignments of $P$. Note that the size of $\cb$ is at most $2^k$.
	Consider the following DNF formula over $\{0,1\}^{kn}$:
	\[
		\psi(\bz)
		= \bigvee_{\bb \in \cb} \bigwedge_{j \in [k]} \bigwedge_{\{l:x_l \neq b_j \}} z_{j,l}~.
	\]
	For a hyperedge $S=(i_1,\ldots,i_k)$, we have
	\begin{align*}
		\psi(\bz^S)=1
		&\iff \exists \bb \in \cb \; \forall j \in [k] \; \forall x_l \neq b_j, \; z^S_{j,l}=1
		\\
		&\iff \exists \bb \in \cb \; \forall j \in [k] \; \forall x_l \neq b_j, \; i_j \neq l
		\\
		&\iff \exists \bb \in \cb \; \forall j \in [k], \; x_{i_j} = b_j
		\\
		&\iff \exists \bb \in \cb, \; \bx_S = \bb
		\\
		&\iff P(\bx_S)=1
		\\
		&\iff P_\bx(\bz^S)=1~.
	\end{align*}
\end{proof}

\begin{lemma} \label{lem:network N1 second layer}
	Let $\bx \in \{0,1\}^n$.
	There exists an affine layer with at most $2^k$ outputs, weights bounded by a constant and bias terms bounded by $n \log(n)$ (for a sufficiently large $n$), such that given an input $\bz^S \in \{0,1\}^{kn}$ for some hyperedge $S$, it satisfies the following: For $S$ with $P_\bx(\bz^S) = 0$ all outputs are at most $-1$, and for $S$ with $P_\bx(\bz^S) = 1$ there exists an output greater or equal to $2$.
\end{lemma}
\begin{proof}
	By \lemref{lem:from P to DNF}, there exists a DNF formula $\varphi_\bx$ over $\{0,1\}^{kn}$ with at most $2^k$ terms, such that $\varphi_\bx(\bz^S) = P_\bx(\bz^S)$. Thus, if $P_\bx(\bz^S) = 0$ then all terms in $\varphi_\bx$ are not satisfied for the input $\bz^S$, and if $P_\bx(\bz^S) = 1$ then there is at least one term in $\varphi_\bx$ which is satisfied for the input $\bz^S$. Therefore, it suffices to construct an affine layer such that for an input $\bz^S$, the $j$-th output will be at most $-1$ if the $j$-th term of $\varphi_\bx$ is not satisfied, and at least $2$ otherwise.
	Each term $C_j$ in $\varphi_\bx$ is a conjunction of positive literals. Let $I_j \subseteq [kn]$ be the indices of these literals. The $j$-th output of the affine layer will be 
	\[
		\left(\sum_{l \in I_j} 3 z^{S}_l\right) - 3 |I_j| + 2~. 
	\]
	Note that if the conjunction $C_j$ holds, then this expression is exactly $3 |I_j| -  3 |I_j| + 2 = 2$, and otherwise it is at most $3 (|I_j| - 1) - 3 |I_j| + 2 = -1$.
	Finally, note that all weights are bounded by $3$ and all bias terms are bounded by $n \log(n)$ (for large enough $n$).
\end{proof}

\begin{lemma} \label{lem:network N1}
	Let $\bx \in \{0,1\}^n$.
	There exists a depth-$2$ neural network $N_1$ with input dimension $kn$, $2kn$ hidden neurons, at most $2^k$ output neurons, and parameter magnitudes bounded by $n^3$ (for a sufficiently large $n$), which satisfies the following. We denote the set of output neurons of $N_1$ by $\ce_1$. Let $\bz' \in \reals^{kn}$ be such that $\Psi(\bz') = \bz^S$ for some hyperedge $S$, and assume that for every $i \in [kn]$ we have $z'_i \not \in \left( c, c + \frac{1}{n^2} \right)$. Then, for $S$ with $P_\bx(\bz^S) = 0$ the inputs to all neurons $\ce_1$ are at most $-1$, and for $S$ with $P_\bx(\bz^S) = 1$ there exists a neuron in $\ce_1$ with input at least $2$. Moreover, only the second layer of $N_1$ depends on $\bx$.
\end{lemma}
\begin{proof}
	First, we construct a depth-$2$ neural network $N_\Psi:\reals^{kn} \rightarrow [0,1]^{kn}$ with a single layer of non-linearity, such that for every $\bz' \in \reals^{kn}$ with $z'_i \not \in (c,c+\frac{1}{n^2})$ for every $i \in [kn]$, we have $N_\Psi(\bz') = \Psi(\bz')$. The network $N_\Psi$ has $2kn$ hidden neurons, and computes $N_\Psi(\bz') = (f(z'_1),\ldots,f(z'_{kn}))$, where $f:\reals \rightarrow [0,1]$ is such that
	\[
		f(t) = n^2 \cdot \left(\left[t - c\right]_+ - \left[t - \left(c + \frac{1}{n^2}\right)\right]_+ \right)~.
	\]
	Note that if $t \leq c$ then $f(t)=0$, if $t \geq c+\frac{1}{n^2}$ then $f(t)=1$, and if $c < t <  c+\frac{1}{n^2}$ then $f(t) \in (0,1)$.
	Also, note that all weights and bias terms can be bounded by $n^2$ (for large enough $n$). Moreover, the network $N_\Psi$ does not depend on $\bx$.

	Let $\bz' \in \reals^{kn}$ such that $\Psi(\bz') = \bz^S$ for some hyperedge $S$, and assume that for every $i \in [kn]$ we have $z'_i \not \in \left( c, c + \frac{1}{n^2} \right)$. For such $\bz'$, we have $N_\Psi(\bz') = \Psi(\bz') = \bz^S$. Hence, it suffices to show that we can construct an affine layer with at most $2^k$ outputs, weights bounded by a constant and bias terms bounded by $n^3$, such that given an input $\bz^S$ it satisfies the following: For $S$ with $P_\bx(\bz^S) = 0$ all outputs are at most $-1$, and for $S$ with $P_\bx(\bz^S) = 1$ there exists an output greater or equal to $2$. We construct such an affine layer in \lemref{lem:network N1 second layer}.
\ignore{
	By \lemref{lem:from P to DNF}, there exists a DNF formula $\varphi_\bx$ over $\{0,1\}^{kn}$ with at most $2^k$ terms, such that $\varphi_\bx(\bz^S) = P_\bx(\bz^S)$. Thus, if $P_\bx(\bz^S) = 0$ then all terms in $\varphi_\bx$ are not satisfied for the input $\bz^S$, and if $P_\bx(\bz^S) = 1$ then there is at least one term in $\varphi_\bx$ which is satisfied for the input $\bz^S$. Therefore, it suffices to construct an affine layer such that for an input $\bz^S$, the $j$-th output will be at most $-1$ if the $j$-th term of $\varphi_\bx$ is not satisfied, and at least $2$ otherwise.
	Each term $C_j$ in $\varphi_\bx$ is a conjunction of positive literals. Let $I_j \subseteq [kn]$ be the indices of these literals.
The $j$-th output of the affine layer will be 
	\[
		\left(\sum_{l \in I_j} 3 z^{S}_l\right) - 3 |I_j| + 2~. 
	\]
	Note that if the conjunction $C_j$ holds, then this expression is exactly $3 |I_j| -  3 |I_j| + 2 = 2$, and otherwise it is at most $3 (|I_j| - 1) - 3 |I_j| + 2 = -1$.
	Finally, note that all weights and bias terms in the network $N_1$ can be bounded by $n^3$ (for large enough $n$).
}%ignore
\end{proof}

\ignore{
The following lemma is from \cite{daniely2021local}, but we give it here fore completeness. 

\begin{lemma}[Rephrased from \cite{daniely2021local}] \label{lem:DNF for encoding check}
	 There exists a DNF formula $\varphi$ over $\{0,1\}^{kn}$ with $k \cdot \frac{n(n-1)}{2} + k + n \cdot \frac{k(k-1)}{2}$ terms such that $\varphi(\bz)=1$ iff $\bz$ is not an encoding of a hyperedge.
\end{lemma}
\begin{proof}
\end{proof}
}%ignore

\begin{lemma}  \label{lem:network N2 second layer}
	There exists an affine layer with $2k + n$ outputs, weights bounded by a constant and bias terms bounded by $n\log(n)$ (for a sufficiently large $n$), such that given an input $\bz \in \{0,1\}^{kn}$, if it is an encoding of a hyperedge then all outputs are at most $-1$, and otherwise there exists an output greater or equal to $2$.
\end{lemma}
\begin{proof}
	Note that $\bz \in \{0,1\}^{kn}$ is not an encoding of a hyperedge iff at least one of the following holds: 
	\begin{enumerate}
		\item At least one of the $k$ size-$n$ slices in $\bz$ contains $0$ more than once. 
		\item At least one of the $k$ size-$n$ slices in $\bz$ does not contain $0$. 
		\item There are two size-$n$ slices in $\bz$ that encode the same index. 
	\end{enumerate}
	We define the outputs of our affine layer as follows. 
	First, we have $k$ outputs that correspond to (1). In order to check whether slice $i \in [k]$ contains $0$ more than once, the output will be $3n - 4 - (\sum_{j \in [n]} 3 z_{i,j})$.
	Second, we have $k$ outputs that correspond to (2): in order to check whether slice $i \in [k]$ does not contain $0$, the output will be $(\sum_{j \in [n]} 3 z_{i,j}) - 3n + 2$. 
	Finally, we have $n$ outputs that correspond to (3): in order to check whether there are two slices that encode the same index $j \in [n]$, the output will be $3k - 4 - (\sum_{i \in [k]} 3 z_{i,j})$.
	Note that all weights are bounded by $3$ and all bias terms are bounded by $n \log(n)$ for large enought $n$.
\end{proof}

\begin{lemma}  \label{lem:network N2}
	There exists a depth-$2$ neural network $N_2$ with input dimension $kn$, at most $2kn$ hidden neurons, $2k  + n$ output neurons, and parameter magnitudes bounded by $n^3$ (for a sufficiently large $n$), which satisfies the following. We denote the set of output neurons of $N_2$ by $\ce_2$. Let $\bz' \in \reals^{kn}$ be such that for every $i \in [kn]$ we have $z'_i \not \in \left(c, c + \frac{1}{n^2} \right)$. If $\Psi(\bz')$ is an encoding of a hyperedge then the inputs to all neurons $\ce_2$ are at most $-1$, and otherwise there exists a neuron in $\ce_2$ with input at least $2$.
\end{lemma}
\begin{proof}
	Let $N_\Psi:\reals^{kn} \rightarrow [0,1]^{kn}$ be the depth-$2$ neural network from the proof of \lemref{lem:network N1}, with a single layer of non-linearity with $2kn$ hidden neurons, and parameter magnitudes bounded by $n^2$, such that for every $\bz' \in \reals^{kn}$ with $z'_i \not \in (c,c+\frac{1}{n^2})$ for every $i \in [kn]$, we have $N_\Psi(\bz') = \Psi(\bz')$.

	Let $\bz' \in \reals^{kn}$ be such that for every $i \in [kn]$ we have $z'_i \not \in \left(c, c + \frac{1}{n^2} \right)$. For such $\bz'$ we have $N_\Psi(\bz') = \Psi(\bz')$. Hence, it suffices to show that we can construct an affine layer with $2k + n$ outputs,  weights bounded by a constant and bias terms bounded by $n^3$, such that given an input $\bz \in \{0,1\}^{kn}$, if it is an encoding of a hyperedge then all outputs are at most $-1$, and otherwise there exists an output greater or equal to $2$. We construct such an affine layer in \lemref{lem:network N2 second layer}.
\ignore{
	Note that $\bz$ is not an encoding of a hyperedge iff at least one of the following holds: 
	\begin{enumerate}
		\item At least one of the $k$ size-$n$ slices in $\bz$ contains $0$ more than once. 
		\item At least one of the $k$ size-$n$ slices in $\bz$ does not contain $0$. 
		\item There are two size-$n$ slices in $\bz$ that encode the same index. 
	\end{enumerate}
	We define the outputs of our affine layer as follows. 
	First, we have 
	%$k \cdot \frac{n(n-1)}{2}$ outputs that corresponds to (1). Note that in order to check whether slice $i \in [k]$ has $0$ in indices $j_1,j_2 \in [n]$, the output will be $2 - 3 z_{i,j_1} - 3 z_{i,j_2}$. 
	$k$ outputs that correspond to (1). In order to check whether slice $i \in [k]$ contains $0$ more than once, the output will be $3n - 4 - (\sum_{j \in [n]} 3 z_{i,j})$.
	Second, we have $k$ outputs that correspond to (2): in order to check whether slice $i \in [k]$ does not contain $0$, the output will be $(\sum_{j \in [n]} 3 z_{i,j}) - 3n + 2$. 
	Finally, we have 
	%$n \cdot \frac{k(k-1)}{2}$ outputs that correspond to (3): in order to check whether slices $i_1,i_2 \in [k]$ encode the same index $j \in [n]$, the output will be $2 - 3 z_{i_1,j} - 3 z_{i_2,j}$.
	$n$ outputs that correspond to (3): in order to check whether there are two slices that encode the same index $j \in [n]$, the output will be $3k - 4 - (\sum_{i \in [k]} 3 z_{i,j})$.

	Moreover, note that all weights and bias terms in the network $N_2$ can be bounded by $n^3$ (for large enough $n$).
}%ignore
\end{proof}

\begin{lemma}  \label{lem:network N3}
	There exists a depth-$2$ neural network $N_3$ with input dimension $kn$, at most $n \log(n)$ hidden neurons, $kn \leq n \log(n)$ output neurons, and parameter magnitudes bounded by $n^3$ (for a sufficiently large $n$), which satisfies the following. We denote the set of output neurons of $N_3$ by $\ce_3$. Let $\bz' \in \reals^{kn}$. If there exists $i \in [kn]$ such that $z'_i \in \left(c, c + \frac{1}{n^2} \right)$ then there exists a neuron in $\ce_3$ with input at least $2$. If for all $i \in [kn]$ we have $z'_i \not \in \left(c - \frac{1}{n^2}, c + \frac{2}{n^2} \right)$ then the inputs to all neurons in $\ce_3$ are at most $-1$.
\end{lemma}
\begin{proof}
	It suffices to construct a univariate depth-$2$ network $f:\reals \to \reals$ with one non-linear layer and a constant number of hidden neurons, such that for every input $z'_i \in (c,c+\frac{1}{n^2})$ we have $f(z'_i) = 2$, and for every $z'_i \not \in (c-\frac{1}{n^2},c+\frac{2}{n^2})$ we have $f(z'_i) = -1$.

	We construct $f$ as follows:
	\begin{align*}
		f(z'_i) = &(3 n^2)\left( \left[z'_i - \left(c-\frac{1}{n^2}\right)\right]_+ - \left[z'_i-c\right]_+ \right) -
		\\
		&(3 n^2)\left( \left[z'_i-\left(c+\frac{1}{n^2}\right)\right]_+ - \left[z'_i-\left(c+\frac{2}{n^2}\right)\right]_+ \right) - 1~.
	\end{align*}
	
	Note that all weights and bias terms are bounded by $n^3$ (for large enough $n$).
\end{proof}

\begin{lemma} \label{lem:tau exists}
	Let $q = \poly(n)$ and $r = \poly(n)$. Then, there exists $\tau=\frac{1}{\poly(n)}$ such that for a sufficiently large $n$, with probability at least 
	%$1 - \frac{1}{n}$ 
	$1 - \exp(-n/2)$
	a vector $\bxi \sim \cn(\zero, \tau^2 I_r)$ satisfies $\norm{\bxi} \leq \frac{1}{q}$.
\end{lemma}
\begin{proof}
	Let $\tau = \frac{1}{q \sqrt{2 r n}}$.
	Every component $\xi_i$ in $\bxi$ has the distribution $\cn(0, \tau^2)$.
	By a standard tail bound of the Gaussian distribution, we have for every $i \in [r]$ and $t \geq 0$ that $\Pr[ \xi_i \geq t ] \leq 2 \exp\left( -\frac{t^2}{2 \tau^2}\right)$. Hence, for $t=\frac{1}{q \sqrt{r}}$, we get
	\[
		\Pr\left[ \xi_i \geq \frac{1}{q\sqrt{r}} \right]
		\leq 2 \exp\left( -\frac{1}{2 \tau^2 q^2 r}\right)
		= 2 \exp\left( -\frac{2 r n q^2}{2 q^2 r}\right)
		= 2 \exp\left( -n\right)~.
	\]
	By the union bound, with probability at least $1 - r \cdot 2 e^{-n}$, we have 
	\[
		\norm{\bxi}^2
		 \leq r \cdot \frac{1}{r q^2}
		 = \frac{1}{q^2}~.
	\]
	Thus, for a sufficiently large $n$, with probability at least 
	%$1-\frac{1}{n}$ 
	$1 - \exp(-n/2)$
	we have $\norm{\bxi} \leq \frac{1}{q}$.
\ignore{
	Note that $\norm{\frac{\bxi}{\tau}}^2$ has the Chi-squared distribution. A concentration bound by Laurent and Massart \citep[Lemma~1]{laurent2000adaptive} implies that for all $t > 0$ we have
	\[
	 	\Pr\left[ \norm{\frac{\bxi}{\tau}}^2 - r \geq 2 \sqrt{rt} + 2t \right] \leq e^{-t}~.
	\]
	Plugging-in $t=\frac{r}{4}$ we get 
	\begin{align*}
		\Pr\left[ \norm{\frac{\bxi}{\tau}}^2 \geq 4r \right]
		&= \Pr\left[ \norm{\frac{\bxi}{\tau}}^2 - r \geq 3r \right]
		\\
		&\leq \Pr\left[ \norm{\frac{\bxi}{\tau}}^2 - r \geq \frac{3}{2}r \right]
		\\
		&= \Pr\left[ \norm{\frac{\bxi}{\tau}}^2 - r \geq  2 \sqrt{r \cdot \frac{r}{4}} + 2 \cdot \frac{r}{4} \right] 
		\\
		&\leq \exp\left( - \frac{r}{4} \right)~.
	\end{align*}
}%ignore	
\end{proof}

\begin{lemma} \label{lem:realizable}
	If $\cs$ is pseudorandom then with probability at least $\frac{39}{40}$ (over $\bxi \sim \cn(\zero, \tau^2 I_p)$ and the i.i.d. inputs $\tilde{\bz}_i \sim \cd$) the examples $(\tilde{\bz}_1, \tilde{y}_1),\ldots,(\tilde{\bz}_{m(n)+n^3},\tilde{y}_{m(n)+n^3})$ returned by the oracle are realized by $\hat{N}$.
\end{lemma}
\begin{proof}
	By our choice of $\tau$, with probability at least $1-\frac{1}{n}$ over $\bxi \sim \cn(\zero, \tau^2 I_p)$, we have $|\xi_j| \leq \frac{1}{10}$ for all $j \in [p]$, and for every $\tilde{\bz}$ with $\norm{\tilde{\bz}} \leq 2n$  
%for each neuron in $\hat{N}$ (including the output neuron), there is a difference of at most $\frac{1}{2}$ between inputs to the neuron in the computations $\hat{N}(\tilde{\bz})$ and $\tilde{N}(\tilde{\bz})$. Also, 
the inputs to the neurons $\ce_1,\ce_2,\ce_3$ in the computation $\hat{N}(\tilde{\bz})$ satisfy Properties~\ref{prop:first} through \ref{prop:last}. We first show that with probability at least $1 - \frac{1}{n}$ all examples $\tilde{\bz}_1, \ldots, \tilde{\bz}_{m(n)+n^3}$ satisfy $\norm{\tilde{\bz}_i} \leq 2n$. Hence, with probability at least $1 - \frac{2}{n}$, Properties~\ref{prop:first} through \ref{prop:last} hold for the computations $\hat{N}(\tilde{\bz}_i)$ for all $i \in [m(n)+n^3]$.
	
	%First, we show that every example $(\tilde{\bz}_i, \tilde{y}_i)$ with $\norm{\tilde{\bz}_i} \leq 2n$ is realized by $\hat{N}$. 
	
	%It remains to show that with probability at least $\frac{1}{40}$ all examples $(\tilde{\bz}_i, \tilde{y}_i)$ for $i=1,\ldots,m(n)+n^3$ satisfy $\norm{\tilde{\bz}_i} \leq 2n$.
	Note that $\norm{\tilde{\bz}_i}^2$ has the Chi-squared distribution. Since $\tilde{\bz}_i$ is of dimension $n^2$, a concentration bound by Laurent and Massart \citep[Lemma~1] {laurent2000adaptive} implies that for all $t > 0$ we have
	\[
	 	\Pr\left[ \norm{\tilde{\bz}_i}^2 - n^2 \geq 2n\sqrt{t} + 2t \right] \leq e^{-t}~.
	\]
	Plugging-in $t=\frac{n^2}{4}$, we get
	\begin{align*}
        		\Pr\left[ \norm{\tilde{\bz}_i}^2  \geq  4n^2 \right]
		&= \Pr\left[ \norm{\tilde{\bz}_i}^2 - n^2  \geq  3n^2 \right] 
		\\
		&\leq \Pr\left[ \norm{\tilde{\bz}_i}^2 - n^2  \geq  \frac{3n^2}{2} \right] 
		\\
		&= \Pr\left[ \norm{\tilde{\bz}_i}^2 - n^2 \geq  2n\sqrt{\frac{n^2}{4}} + 2 \cdot \frac{n^2}{4} \right] 
		\\
		&\leq \exp\left(-\frac{n^2}{4}\right)~.
	\end{align*}
	Thus, we have
	$\Pr\left[ \norm{\tilde{\bz}_i}  \geq 2n \right] \leq \exp\left(-\frac{n^2}{4}\right)$.
\ignore{
	\begin{equation*} 
		\Pr\left[ \norm{\tilde{\bz}_i}  \geq 2n \right] 
		\leq \exp\left(-\frac{n^2}{4}\right)~.
	\end{equation*}
}%ignore	
	By the union bound, with probability at least 
	\[	
		1 - \left( m(n)+n^3  \right)  \exp\left(-\frac{n^2}{4}\right) \geq 1 - \frac{1}{n}
	\]
	(for a sufficiently large $n$), all examples $(\tilde{\bz}_i, \tilde{y}_i)$ satisfy $\norm{\tilde{\bz}_i} \leq 2n$.
	
	Thus, we showed that with probability at least $1 - \frac{2}{n} \geq \frac{39}{40}$ (for a sufficiently large $n$), we have $|\xi_j| \leq \frac{1}{10}$ for all $j \in [p]$, and Properties~\ref{prop:first} through \ref{prop:last} hold for the computations $\hat{N}(\tilde{\bz}_i)$ for all $i \in [m(n)+n^3]$. It remains to show that if these properties hold, then the examples $(\tilde{\bz}_1, \tilde{y}_1),\ldots,(\tilde{\bz}_{m(n)+n^3},\tilde{y}_{m(n)+n^3})$ are realized by $\hat{N}$.
	%Let $\tilde{\bz} \in \reals^{n^2}$, let $\bz' = \tilde{\bz}_[kn]$, and suppose that Properties~\ref{prop:first} through \ref{prop:last} hold for the computation $\hat{N}(\tilde{\bz})$.
	
	Let $i \in [m(n)+n^3]$. For brevity, we denote $\tilde{\bz} = \tilde{\bz}_i$, $\tilde{y} = \tilde{y}_i$, and $\bz' = \tilde{\bz}_{[kn]}$.
	Since $|\xi_j| \leq \frac{1}{10}$ for all $j \in [p]$, and all incoming weights to the output neuron in $\tilde{N}$ are $-1$, then in $\hat{N}$ all incoming weights to the output neuron are in $\left[ - \frac{11}{10}, -\frac{9}{10} \right]$, and the bias term in the output neuron, denoted by $\hat{b}$, is in $\left[ \frac{9}{10}, \frac{11}{10} \right]$. 
	Consider the following cases:
	\begin{itemize}
		\item If $\Psi(\bz')$ is not an encoding of a hyperedge then $\tilde{y} = 0$, and $\hat{N}(\tilde{\bz})$ satisfies:
			\begin{enumerate}
				\item If $\bz'$ does not have components in $\left( c, c+\frac{1}{n} \right)$, then there exists a neuron in $\ce_2$ with output at least $\frac{3}{2}$  (by Property~\ref{prop:second}).
				\item If $\bz'$ has a component in $\left( c, c+\frac{1}{n} \right)$, then there exists a neuron in $\ce_3$ with output at least $\frac{3}{2}$ (by Property~\ref{prop:last}). 
			\end{enumerate}
			In both cases, since all incoming weights to the output neuron in $\hat{N}$ are in $\left[ - \frac{11}{10}, -\frac{9}{10} \right]$, and $\hat{b} \in \left[ \frac{9}{10}, \frac{11}{10} \right]$, then the input to the output neuron (including the bias term) is at most $\frac{11}{10} - \frac{3}{2} \cdot \frac{9}{10} < 0$, and thus its output is $0$.
\ignore{
		\begin{itemize}
			\item If $\bz'$ does not have components in $\left( c, c+\frac{1}{n} \right)$, then there exists a neuron in $\ce_2$ with output at least $\frac{3}{2}$. Since all incoming weights to the output neuron in $\hat{N}$ are in $\left[ - \frac{11}{10}, -\frac{9}{10} \right]$, and $\hat{b} \in \left[ \frac{9}{10}, \frac{11}{10} \right]$, then the input to output neuron (including the bias term but without the activation) is at most $\frac{11}{10} - \frac{3}{2} \cdot \frac{9}{10} < 0$, and thus its output is $0$.
			\item If $\bz'$ has a component in $\left( c, c+\frac{1}{n} \right)$, then there exists a neuron in $\ce_3$ with input at least $\frac{3}{2}$.  Since all incoming weights to the output neuron in $\hat{N}$ are in $\left[ - \frac{11}{10}, -\frac{9}{10} \right]$, and $\hat{b} \in \left[ \frac{9}{10}, \frac{11}{10} \right]$, then the input to output neuron (including the bias term but without the activation) is at most $\frac{11}{10} - \frac{3}{2} \cdot \frac{9}{10} < 0$, and thus its output is $0$.
		\end{itemize} 	
}%ignore
	\item If $\Psi(\bz')$ is an encoding of a hyperedge $S$, then by the definition of the examples oracle we have $S=S_i$. Hence:
		\begin{itemize}
			
			\item If $\bz'$ does not have components in $\left( c - \frac{1}{n^2}, c + \frac{2}{n^2} \right)$, then:
			
			\begin{itemize}
				
				\item	If $y_i = 0$ then the oracle sets $\tilde{y}=\hat{b}$. Since $\cs$ is pseudorandom, we have $P_\bx(\bz^S) = P_\bx(\bz^{S_i}) = y_i = 0$. 
				%If $P_\bx(\bz^S)=0$: Since $\cs$ is pseudorandom, we have $y_i = P_\bx(\bz^{S_i}) = P_\bx(\bz^S)=0$. Thus, the oracle sets $\tilde{y}=\hat{b}$. Also,
				Hence,
				in the computation $\hat{N}(\tilde{\bz})$ the inputs to all neurons in $\ce_1,\ce_2,\ce_3$ are at most $-\frac{1}{2}$ (by Properties~\ref{prop:first},~\ref{prop:second} and~\ref{prop:last}), and thus their outputs are $0$. Therefore, $\hat{N}(\tilde{\bz}) = \hat{b}$.
				
				\item If $y_i = 1$ then the oracle sets $\tilde{y}=0$. Since $\cs$ is pseudorandom, we have $P_\bx(\bz^S) = P_\bx(\bz^{S_i}) = y_i = 1$. 
				%If $P_\bx(\bz^S)=1$: Since $\cs$ is pseudorandom, we have $y_i = P_\bx(\bz^{S_i}) = P_\bx(\bz^S)=1$. Thus, we have $\tilde{y}=0$. Also, 
				Hence,
				in the computation $\hat{N}(\tilde{\bz})$ there exists a neuron in $\ce_1$ with output at least $\frac{3}{2}$  (by Property~\ref{prop:first}). Since all incoming weights to the output neuron in $\hat{N}$ are in $\left[ - \frac{11}{10}, -\frac{9}{10} \right]$, and $\hat{b} \in \left[ \frac{9}{10}, \frac{11}{10} \right]$, then the input to output neuron (including the bias term) is at most $\frac{11}{10} - \frac{3}{2} \cdot \frac{9}{10} < 0$, and thus its output is $0$.
			
			\end{itemize} 
			
			\item If $\bz'$ has a component in $\left( c , c + \frac{1}{n^2} \right)$, then $\tilde{y}=0$. Also, in the computation $\hat{N}(\tilde{\bz})$ there exists a neuron in $\ce_3$ with output at least $\frac{3}{2}$ (by Property~\ref{prop:last}). Since all incoming weights to the output neuron in $\hat{N}$ are in $\left[ - \frac{11}{10}, -\frac{9}{10} \right]$, and $\hat{b} \in \left[ \frac{9}{10}, \frac{11}{10} \right]$, then the input to output neuron (including the bias term) is at most $\frac{11}{10} - \frac{3}{2} \cdot \frac{9}{10} < 0$, and thus its output is $0$.
			
			\item  If $\bz'$ does not have components in the interval $(c,c+\frac{1}{n^2})$, but has a component in the interval $(c-\frac{1}{n^2},c+\frac{2}{n^2})$, then:
			
			\begin{itemize}
				
				\item If $y_i = 1$ the oracle sets $\tilde{y}=0$. Since $\cs$ is pseudorandom, we have $P_\bx(\bz^S) = P_\bx(\bz^{S_i}) = y_i = 1$. 
				%If $P_\bx(\bz^S)=1$: Since $\cs$ is pseudorandom, we have $y_i = P_\bx(\bz^{S_i}) = P_\bx(\bz^S)=1$.  Hence, $\tilde{y}=0$. Also, 
				Hence, in the computation $\hat{N}(\tilde{\bz})$ there exists a neuron in $\ce_1$ with output at least $\frac{3}{2}$ (by Property~\ref{prop:first}). Since all incoming weights to the output neuron in $\hat{N}$ are in $\left[ - \frac{11}{10}, -\frac{9}{10} \right]$, and $\hat{b} \in \left[ \frac{9}{10}, \frac{11}{10} \right]$, then the input to output neuron (including the bias term) is at most $\frac{11}{10} - \frac{3}{2} \cdot \frac{9}{10} < 0$, and thus its output is $0$. 
				
				\item  If $y_i = 0$ the oracle sets $\tilde{y}=[\hat{b} - \hat{N}_3(\tilde{\bz})]_+$. Since $\cs$ is pseudorandom, we have $P_\bx(\bz^S) = P_\bx(\bz^{S_i}) = y_i = 0$. 
				%If $P_\bx(\bz^S)=0$: Since $\cs$ is pseudorandom, we have $y_i = P_\bx(\bz^{S_i}) = P_\bx(\bz^S)=0$.  Hence, $\tilde{y}=[\hat{b} - \hat{N}_3(\tilde{\bz})]_+$. Also,
				Therefore, in the computation $\hat{N}(\tilde{\bz})$ all neurons in $\ce_1,\ce_2$ have output $0$ (by Properties~\ref{prop:first} and~\ref{prop:second}), and hence their contribution to the output of $\hat{N}$ is $0$. Thus, by the definition of $\hat{N}_3$, we have $\hat{N}(\tilde{\bz}) = [\hat{b} - \hat{N}_3(\tilde{\bz})]_+$.
			
			\end{itemize}			
		
		\end{itemize}
	
	\end{itemize}
\end{proof}

\begin{lemma} \label{lem:pseudorandom small loss}
	If $\cs$ is pseudorandom, then for a sufficiently large $n$, with probability greater than $\frac{2}{3}$ we have
	\[
		\ell_I(h') \leq \frac{2}{n}~.
	\]
\end{lemma}
\begin{proof}
	By \lemref{lem:realizable}, 
%we show that 
if $\cs$ is pseudorandom then with probability at least $\frac{39}{40}$ (over $\bxi \sim \cn(\zero, \tau^2 I_p)$ and the i.i.d. inputs $\tilde{\bz}_i \sim \cd$) the examples $(\tilde{\bz}_1, \tilde{y}_1),\ldots,(\tilde{\bz}_{m(n)},\tilde{y}_{m(n)})$ returned by the oracle are realized by $\hat{N}$. 
%(where the subnetwork $N_1$ is defined w.r.t. the ransom $\bx$ that corresponds to the PRG).
Recall that the algorithm $\cl$ is such that with probability at least $\frac{3}{4}$ (over $\bxi \sim \cn(\zero, \tau^2 I_p)$, the i.i.d. inputs $\tilde{\bz}_i \sim \cd$, and possibly its internal randomness), given a size-$m(n)$ dataset labeled by $\hat{N}$, it returns a hypothesis $h$ such that $\E_{\tilde{\bz} \sim \cd} \left[(h(\tilde{\bz})-\hat{N}(\tilde{\bz}))^2 \right] \leq \frac{1}{n}$.
Hence, with probability at least $\frac{3}{4} - \frac{1}{40}$ the algorithm $\cl$ returns such a good hypothesis $h$, given $m(n)$ examples labeled by our examples oracle. 
Indeed, note that $\cl$ can return a bad hypothesis only if the random choices are either bad for $\cl$ (when used with realizable examples) or bad for the realizability of the examples returned by our oracle.
By the definition of $h'$ and the construction of $\hat{N}$, if $h$ has small error then $h'$ also has small error, namely, 
\[
	 \E_{\tilde{\bz} \sim \cd} \left[(h'(\tilde{\bz})-\hat{N}(\tilde{\bz}))^2 \right] 
	 \leq \E_{\tilde{\bz} \sim \cd} \left[(h(\tilde{\bz})-\hat{N}(\tilde{\bz}))^2 \right] 
	 \leq \frac{1}{n}~.
\]

Let $\hat{\ell}_{I}(h')=\frac{1}{|I|}\sum_{i \in I}(h'(\tilde{\bz}_i)-\hat{N}(\tilde{\bz}_i))^2$.
Recall that by our choice of $\tau$ we have $\Pr[\hat{b} > \frac{11}{10}] \leq \frac{1}{n}$.
Since, $(h'(\tilde{\bz})-\hat{N}(\tilde{\bz}))^2 \in [0,\hat{b}^2]$ for all $\tilde{\bz} \in \reals^{n^2}$,
by Hoeffding's inequality, we have for a sufficiently large $n$ that
\begin{align*}
	\Pr\left[\left|\hat{\ell}_{I}(h') -  \E_{\tilde{\cs}_I}\hat{\ell}_{I}(h')\right| \geq \frac{1}{n}\right]
	&= \Pr\left[\left|\hat{\ell}_{I}(h') -  \E_{\tilde{\cs}_I}\hat{\ell}_{I}(h')\middle| \geq \frac{1}{n} \right| \hat{b} \leq \frac{11}{10}\right] \cdot \Pr\left[ \hat{b} \leq \frac{11}{10} \right]
	\\
	&\;\;\;\; + \Pr\left[\left|\hat{\ell}_{I}(h') -  \E_{\tilde{\cs}_I}\hat{\ell}_{I}(h')\middle| \geq \frac{1}{n} \right| \hat{b} > \frac{11}{10}\right] \cdot \Pr\left[ \hat{b} > \frac{11}{10} \right]
	\\
	&\leq 2 \exp\left( -\frac{2n^3}{n^2 (11/10)^4} \right) \cdot 1 + 1 \cdot \frac{1}{n}
	\\
	&\leq \frac{1}{40}~.
\end{align*}
Moreover, by \lemref{lem:realizable},
\[
	\Pr \left[  \ell_I(h') \neq \hat{\ell}_I(h') \right]
	\leq \Pr \left[ \exists i \in I \text{ s.t. } \tilde{y}_i \neq \hat{N}(\tilde{\bz}_i) \right]
	\leq \frac{1}{40}~. 
\]

Overall, by the union bound we have with probability at least $1-\left( \frac{1}{4} + \frac{1}{40} + \frac{1}{40} + \frac{1}{40}\right)  > \frac{2}{3}$ for sufficiently large $n$ that:
\begin{itemize}
	\item $\E_{\tilde{\cs}_I}\hat{\ell}_{I}(h') = \E_{\tilde{\bz} \sim \cd} \left[(h'(\tilde{\bz})-\hat{N}(\tilde{\bz}))^2 \right] \leq \frac{1}{n}$.
%	\item $\hat{b} \leq \frac{11}{10}$.
	\item $\left|\hat{\ell}_{I}(h') -  \E_{\tilde{\cs}_I}\hat{\ell}_{I}(h')\right| \leq \frac{1}{n}$.
	\item $\ell_I(h') - \hat{\ell}_I(h') = 0$.
\end{itemize}
Combining the above, we get that if $\cs$ is pseudorandom, then with probability greater than $\frac{2}{3}$ we have
\[
	\ell_I(h')
	= \left( \ell_I(h') - \hat{\ell}_I(h') \right) + \left( \hat{\ell}_I(h') - \E_{\tilde{\cs}_I}\hat{\ell}_{I}(h') \right) +\E_{\tilde{\cs}_I}\hat{\ell}_{I}(h')  
	\leq 0 + \frac{1}{n} + \frac{1}{n} 
	= \frac{2}{n}~.
\]
\end{proof}

\begin{lemma} \label{lem:prob z good discrete}
	Let $\bz \in \{0,1\}^{kn}$ be a random vector whose components are drawn i.i.d. from a Bernoulli distribution, which takes the value $0$ with probability $\frac{1}{n}$. Then, for a sufficiently large $n$, the vector $\bz$ is an encoding of a hyperedge with probability at least $\frac{1}{\log(n)}$.
\end{lemma}
\begin{proof}
	The vector $\bz$ represents a hyperedge iff in each of the $k$ size-$n$ slices in $\bz$ there is exactly one $0$-bit and each two of the $k$ slices in $\bz$ encode different indices. Hence, 
	\begin{align*}
		 \Pr\left[\bz \text{ represents a hyperedge} \right]
		&=n \cdot (n-1) \cdot \ldots \cdot(n-k+1) \cdot \left(\frac{1}{n}\right)^k \left(\frac{n-1}{n}\right)^{nk-k}
		\\
		&\geq \left(\frac{n-k}{n}\right)^k \left(\frac{n-1}{n}\right)^{k(n-1)}
		\\
		&=\left(1-\frac{k}{n}\right)^k \left(1-\frac{1}{n}\right)^{k(n-1)}~.
	\end{align*}
	Since for every $x \in (0,1)$ we have $e^{-x} < 1 - \frac{x}{2}$, then for a sufficiently large $n$ the above is at least
	\begin{equation*} \label{eq:prob represents hyperedge}
		\exp\left(-\frac{2k^2}{n}\right) \cdot \exp\left(-\frac{2k(n-1)}{n} \right)
		\geq \exp\left(-1\right) \cdot \exp\left(-2k\right)
		\geq \frac{1}{\log(n)}~.
	\end{equation*}
\end{proof}

\begin{lemma} \label{lem:prob z good}
	Let $\tilde{\bz} \in \reals^{n^2}$ be the vector returned by the oracle. We have
	\[
		\Pr\left[\tilde{\bz} \in \tilde{\cz}\right] \geq \frac{1}{2\log(n)}~.
	\]
\end{lemma}
\begin{proof}
	Let $\bz' = \tilde{\bz}_{[kn]}$. We have
	\begin{equation} \label{eq:prob z good}
		\Pr\left[\tilde{\bz} \not \in \tilde{\cz}\right] 
		\leq \Pr\left[ \exists j \in [kn] \text{ s.t. } z'_j \in \left(c-\frac{1}{n^2},c+\frac{2}{n^2}\right) \right] +  \Pr\left[\Psi(\bz') \text{ does not represent a hyperedge} \right]~.
	\end{equation}
	
	We now bound the terms in the above RHS.
	First, since $\bz'$ has the Gaussian distribution, then its components are drawn i.i.d. from a density function bounded by $\frac{1}{2\pi}$. Hence, for a sufficiently large $n$ we have
	\begin{equation} \label{eq:prob good components}
		\Pr\left[ \exists j \in [kn] \text{ s.t. } z'_j \in \left(c-\frac{1}{n^2},c+\frac{2}{n^2}\right) \right] 
		\leq kn \cdot \frac{1}{2\pi} \cdot \frac{3}{n^2}
		= \frac{3k}{2 \pi n}
		\leq \frac{\log(n)}{n}~.
	\end{equation}
	
	Let $\bz = \Psi(\bz')$. Note that $\bz$ is a random vector whose components are drawn i.i.d. from a Bernoulli distribution, where the probability to get $0$ is $\frac{1}{n}$. By \lemref{lem:prob z good discrete}, $\bz$ is an encoding of a hyperedge with probability at least $\frac{1}{\log(n)}$. Combining it with \eqref{eq:prob z good} and~(\ref{eq:prob good components}), , we get for a sufficiently large $n$ that
	\[
		\Pr\left[\tilde{\bz} \not \in \tilde{\cz}\right] 
		\leq \frac{\log(n)}{n} + \left( 1 - \frac{1}{\log(n)} \right)
		\leq 1 - \frac{1}{2 \log(n)}~,
	\]
	as required.
\end{proof}

\begin{lemma} \label{lem:random large loss}
	 If $\cs$ is random, then for a sufficiently large $n$ with probability larger than $\frac{2}{3}$ we have
	 \[
    		\ell_I(h') > \frac{2}{n}~.
    	\]
\end{lemma}
\begin{proof}
        Let $\tilde{\cz} \subseteq \reals^{n^2}$ be such that $\tilde{\bz} \in \tilde{\cz}$ iff $\tilde{\bz}_{[kn]}$ does not have components in the interval $(c-\frac{1}{n^2},c+\frac{2}{n^2})$, and $\Psi(\tilde{\bz}_{[kn]})=\bz^{S}$ for a hyperedge $S$.
        If $\cs$ is random, then by the definition of our examples oracle, for every $i \in [m(n)+n^3]$ such that $\tilde{\bz}_i \in \tilde{\cz}$, we have $\tilde{y}_i=\hat{b}$ with probability $\frac{1}{2}$ and $\tilde{y}_i=0$ otherwise. Also, by the definition of the oracle, $\tilde{y}_i$ is independent of $S_i$ and independent of the choice of the vector $\tilde{\bz}_i$ that corresponds to $\bz^{S_i}$.
        If $\hat{b} \geq \frac{9}{10}$ then for a sufficiently large $n$ the hypothesis $h'$ satisfies for each random example $(\tilde{\bz}_i,\tilde{y}_i) \in \tilde{\cs}_I$ the following
        \begin{align*} \label{eq:nn-large error}
        	\Pr_{(\tilde{\bz}_i,\tilde{y}_i)}&\left[(h'(\tilde{\bz}_i)-\tilde{y}_i)^2 \geq \frac{1}{5}\right]
        	\\
        	&\geq \Pr_{(\tilde{\bz}_i,\tilde{y}_i)} \left[\left.(h'(\tilde{\bz}_i)-\tilde{y}_i)^2 \geq \frac{1}{5} \; \right| \; \tilde{\bz}_i \in \tilde{\cz} \right] \cdot \Pr_{\tilde{\bz}_i} \left[\tilde{\bz}_i \in \tilde{\cz} \right]
        	\\
        	&\geq  \Pr_{(\tilde{\bz}_i,\tilde{y}_i)} \left[\left.(h'(\tilde{\bz}_i)-\tilde{y}_i)^2 \geq \left(\frac{\hat{b}}{2}\right)^2 \; \right| \; \tilde{\bz}_i \in \tilde{\cz} \right] \cdot \Pr_{\tilde{\bz}_i} \left[\tilde{\bz}_i \in \tilde{\cz}\right] 
        	\\
        	&\geq \frac{1}{2}  \cdot \Pr_{\tilde{\bz}_i} \left[\tilde{\bz}_i \in \tilde{\cz}\right]~.
        \end{align*}
        In \lemref{lem:prob z good}, we show that $\Pr_{\tilde{\bz}_i} \left[\tilde{\bz}_i \in \tilde{\cz}\right] \geq \frac{1}{2\log(n)}$.
        Hence,
        \begin{align*}
        	\Pr_{(\tilde{\bz}_i,\tilde{y}_i)}  \left[(h'(\tilde{\bz}_i)-\tilde{y}_i)^2 \geq \frac{1}{5}\right]
        	\geq  \frac{1}{2}  \cdot \frac{1}{2\log(n)}
        	\geq \frac{1}{4\log(n)}~.
        \end{align*}
        Thus, if $\hat{b} \geq \frac{9}{10}$ then we have
        \[
        	\E_{\tilde{\cs}_I}\left[ \ell_I(h') \right] \geq \frac{1}{5} \cdot \frac{1}{4\log(n)} = \frac{1}{20\log(n)}~.
        \]
        Therefore, for large $n$ we have
        \[
        	\Pr\left[ \E_{\tilde{\cs}_I}\left[ \ell_I(h') \right] \geq \frac{1}{20\log(n)} \right] \geq 1-\frac{1}{n} \geq \frac{7}{8}~.
        \]
        
        Since, $(h'(\tilde{\bz})-\tilde{y})^2 \in [0,\hat{b}^2]$ for all $\tilde{\bz},\tilde{y}$ returned by the examples oracle, and the examples $\tilde{\bz}_i$ for $i \in I$ are i.i.d., then by Hoeffding's inequality, we have for a sufficiently large $n$ that
        \begin{align*}
        	\Pr\left[\left|\ell_{I}(h') -  \E_{\tilde{\cs}_I}\ell_{I}(h')\right| \geq \frac{1}{n}\right]
        	&= \Pr\left[\left|\ell_{I}(h') -  \E_{\tilde{\cs}_I}\ell_{I}(h')\middle| \geq \frac{1}{n} \right| \hat{b} \leq \frac{11}{10}\right] \cdot \Pr\left[ \hat{b} \leq \frac{11}{10} \right]
        	\\
        	&\;\;\;\; + \Pr\left[\left|\ell_{I}(h') -  \E_{\tilde{\cs}_I}\ell_{I}(h')\middle| \geq \frac{1}{n} \right| \hat{b} > \frac{11}{10}\right] \cdot \Pr\left[ \hat{b} > \frac{11}{10} \right]
        	\\
        	&\leq 2 \exp\left( -\frac{2n^3}{n^2 (11/10)^4} \right) \cdot 1 + 1 \cdot \frac{1}{n}
        	\\
        	&\leq \frac{1}{8}~.
        \end{align*}
        
        Hence, for large enough $n$, with probability at least $1 - \frac{1}{8} - \frac{1}{8} = \frac{3}{4} > \frac{2}{3}$ we have both $ \E_{\tilde{\cs}_I}\left[ \ell_I(h') \right] \geq \frac{1}{20\log(n)}$ and $\left|\ell_{I}(h') -  \E_{\tilde{\cs}_I}\ell_{I}(h')\right| \leq \frac{1}{n}$, and thus
        \[
        	\ell_I(h') \geq \frac{1}{20\log(n)} - \frac{1}{n} > \frac{2}{n}~.
        \]
\end{proof}

%ignore

\section{Proof of \corollaryref{cor:non-degenerate}} \label{app:proof of cor}

By the proof of  \thmref{thm:hard smoothed}, under \assref{ass:localPRG}, there is no $\poly(d)$-time algorithm $\cl_s$ that satisfies the following: Let $\btheta \in \reals^p$ be $B$-bounded parameters of a depth-$3$ network $N_\btheta:\reals^d \to \reals$, and let $\tau,\epsilon > 0$. Assume that $p,B,1/\epsilon,1/\tau \leq \poly(d)$, and that the widths of the hidden layers in $\cn_\btheta$ are $d$ (i.e., the weight matrices are square). Let $\bxi \in \cn(\zero, \tau^2 I_p)$ and let $\hat{\btheta} = \btheta + \bxi$. Then, with probability at least $\frac{3}{4} - \frac{1}{1000}$, given access to an examples oracle for $\cn_{\hat{\btheta}}$, the algorithm $\cl_s$ returns a hypothesis $h$ with $\E_{\bx} \left[ (h(\bx) - N_{\hat{\btheta}})^2 \right] \leq \epsilon$. 

Note that in the above, the requirements from $\cl_s$ are somewhat weaker than in our original definition of learning with smoothed parameters. Indeed, we assume that the widths of the hidden layers are $d$ and the required success probability is only $\frac{3}{4} - \frac{1}{1000}$ (rather than $\frac{3}{4}$). We now explain why the hardness result holds already under these conditions:
\begin{itemize}
	\item Note that if we change the assumption on the learning algorithm in proof of  \thmref{thm:hard smoothed} such that it succeeds with probability at least $\frac{3}{4} - \frac{1}{1000}$ (rather than $\frac{3}{4}$), then in the case where $\cs$ is pseudorandom we get that the algorithm $\ca$ returns $1$ with probability at least $1 - \left(\frac{1}{4} + \frac{1}{1000} + \frac{1}{40} + \frac{1}{40} + \frac{1}{40}\right)$ (see the proof of \lemref{lem:pseudorandom small loss}), which is still greater than $\frac{2}{3}$. Also, the analysis of the case where $\cs$ is random does not change, and thus in this case $\ca$ returns $0$ with probability greater than $\frac{2}{3}$. Consequently, we still get distinguishing advantage greater than $\frac13$.
	\item Regarding the requirement on the widths, we note that in the proof of \thmref{thm:hard smoothed} the layers satisfy the following. The input dimension is $d=n^2$, the width of the first hidden layer is at most $3n\log(n) \leq d$, and the width of the second hidden layer is at most $\log(n) + 2n + n \log(n) \leq d$ (all bounds are for a sufficiently large $d$). In order to get a network where all layers are of width $d$, we add new neurons to the hidden layers, with incoming weights $0$, outgoing weights $0$, and bias terms $-1$. Then, for an appropriate choice of $\tau=1/\poly(n)$, even in the perturbed network the outputs of these new neurons will be $0$ w.h.p. for every input $\tilde{\bz}_1,\ldots,\tilde{\bz}_{m(n)+n^3}$, and thus they will not affect the network's output. Thus, using the same argument as in the proof of \thmref{thm:hard smoothed}, we conclude that the hardness results holds already for network with square weight matrices.
\end{itemize}

Suppose that there exists an efficient algorithm $\cl_p$ that learns in the standard PAC framework depth-$3$ neural networks where the minimal singular value of each weight matrix is lower bounded by $1/q(d)$ for any polynomial $q(d)$. We will use $\cl_p$ to obtain an efficient algorithm $\cl_s$ that learns depth-$3$ networks with smoothed parameters as described above, and thus reach a contradiction.

Let $\btheta \in \reals^p$ be $B$-bounded parameters of a depth-$3$ network $N_\btheta:\reals^d \to \reals$, and let $\tau,\epsilon > 0$. Assume that $p,B,1/\epsilon,1/\tau \leq \poly(d)$, and that the widths of the hidden layers in $\cn_\btheta$ are $d$.
For random $\bxi \sim \cn(\zero, \tau^2 I_p)$ and $\hat{\btheta} = \btheta + \bxi$, the algorithm $\cl_s$ has access to examples labeled by $N_{\hat{\btheta}}$. 
Using \lemref{lem:min singular} below with $t = \frac{\tau}{d}$ and the union bound over the two weight matrices in $N_\btheta$, we get that with probability at least $1 - \frac{2 \cdot 2.35}{\sqrt{d}} \geq 1 - \frac{1}{1000}$ (for large enough $d$), the minimal singular values of all weight matrices in $\hat{\btheta}$ are at least $\frac{\tau}{d} \geq \frac{1}{q(d)}$ for some sufficiently large polynomial $q(d)$.
Our algorithm $\cl_s$ will simply run $\cl_p$. Given that the minimal singular values of the weight matrices are at least $\frac{1}{q(d)}$, the algorithm $\cl_p$ runs in time $\poly(d)$ and returns with probability at least $\frac{3}{4}$ 
%(over its internal randomness) 
a hypothesis $h$ with $\E_{\bx} \left[ (h(\bx) - N_{\hat{\btheta}}(\bx))^2 \right] \leq \epsilon$. Overall, the algorithm $\cl_s$ runs in $\poly(d)$ time, and with probability at least $\frac{3}{4} - \frac{1}{1000}$ (over both $\bxi$ and the internal randomness) returns a hypothesis $h$ with loss at most $\epsilon$.

\begin{lemma}[\cite{sankar2006smoothed}, Theorem~3.3] \label{lem:min singular}
	Let $W$ be an arbitrary square matrix in $\reals^{d \times d}$, and let $P \in \reals^{d \times d}$ be a random matrix, where each entry is drawn i.i.d. from $\cn(0, \tau^2)$ for some $\tau > 0$. Let $\sigma_d$ be the minimal singular value of the matrix $W+P$. Then, for every $t > 0$ we have
	\[
		\Pr_P\left[ \sigma_d \leq t  \right] \leq 2.35 \cdot \frac{t\sqrt{d}}{\tau}~.
	\]
\end{lemma}

\section{Proof of \thmref{thm:smoothed weights and inputs}} \label{app:proof of smoothed weights and inputs}

The proof follows similar ideas to the proof of \thmref{thm:hard smoothed}. The main difference is that we need to handle here a smoothed discrete input distribution rather than the standard Gaussian distribution.

For a sufficiently large $n$, let $\cd$ be a distribution on $\{0,1\}^{n^2}$, where each component is drawn i.i.d. from a Bernoulli distribution which takes the value $0$ with probability $\frac{1}{n}$. Assume that there is a $\poly(n)$-time algorithm $\cl$ that learns depth-$3$ neural networks with at most $n^2$ hidden neurons and parameter magnitudes bounded by $n^3$, with smoothed parameters and inputs, under the distribution $\cd$, with $\epsilon=\frac{1}{n}$ and $\tau,\omega=1/\poly(n)$ that we will specify later. Let $m(n) \leq \poly(n)$ be the sample complexity of $\cl$, namely, $\cl$ uses a sample of size at most $m(n)$ and returns with probability at least $\frac34$ a hypothesis $h$ with $\E_{\bz \sim \hat{\cd}} \left[ \left(h(\bz) - N_{\hat{\btheta}}(\bz) \right)^2 \right] \leq \epsilon = \frac{1}{n}$. Note that $\hat{\cd}$ is the distribution $\cd$ after smoothing with parameter $\omega$, and the vector $\hat{\btheta}$ is the parameters of the target network after smoothing with parameter $\tau$. Let $s>1$ be a constant such that $n^s \geq m(n) + n^3$ for every sufficiently large $n$. By \assref{ass:localPRG}, there exists a constant $k$ and a predicate $P:\{0,1\}^k \to \{0,1\}$, such that $\cf_{P,n,n^s}$ is $\frac{1}{3}$-PRG. We will show an efficient algorithm $\ca$ with distinguishing advantage greater than $\frac13$ and thus reach a contradiction.

Throughout this proof, we will use some notations from the proof of \thmref{thm:hard smoothed}. We repeat it here for convenience.
For a hyperedge $S = (i_1,\ldots,i_k)$ we denote by $\bz^S \in \{0,1\}^{kn}$ the following encoding of $S$: the vector $\bz^S$ is a concatenation of $k$ vectors in $\{0,1\}^n$, such that the $j$-th vector has $0$ in the $i_j$-th coordinate and $1$ elsewhere. Thus, $\bz^S$ consists of $k$ size-$n$ slices, each encoding a member of $S$. For $\bz \in \{0,1\}^{kn}$, $i \in [k]$ and $j \in [n]$, we denote $z_{i,j} = z_{(i-1)n + j}$. That is, $z_{i,j}$ is the $j$-th component in the $i$-th slice in $\bz$. For $\bx \in \{0,1\}^n$, let $P_\bx: \{0,1\}^{kn} \to \{0,1\}$ be such that for every hyperedge $S$ we have $P_\bx(\bz^S) = P(\bx_S)$. 
%Let $\Psi: \reals^{kn} \to \{0,1\}^{kn}$ be a mapping such that for every $\bz' \in \reals^{kn}$ and $i \in [kn]$ we have $\Psi(\bz')_i = \onefunc[z'_i \geq 1/2]$. 
For $\tilde{\bz} \in \reals^{n^2}$ we denote $\tilde{\bz}_{[kn]} = (\tilde{z}_1,\ldots,\tilde{z}_{kn})$, namely, the first $kn$ components of $\tilde{\bz}$ (assuming $n^2 \geq kn$).

\subsection{Defining the target network for $\cl$}

Since our goal is to use the algorithm $\cl$ for breaking PRGs, in this subsection we define a neural network $\tilde{N}:\reals^{n^2} \to \reals$ that we will later use as a target network for $\cl$. The network $\tilde{N}$ contains the subnetworks $N_1,N_2$ that we define below.

Let $N_1$ be a depth-$1$ neural network (i.e., one layer, with activations in the output neurons) with input dimension $kn$, at most $\log(n)$ output neurons, and parameter magnitudes bounded by $n^3$ (all bounds are for a sufficiently large $n$), which satisfies the following. We denote the set of output neurons of $N_1$ by $\ce_1$.
Let $\bz' \in \{0,1\}^{kn}$ be an input to $N_1$ such that $\bz' = \bz^S$ for some hyperedge $S$. Thus, even though $N_1$ takes inputs in $\reals^{kn}$, we consider now its behavior for an input $\bz'$ with discrete components in $\{0,1\}$. Fix some $\bx \in \{0,1\}^n$. Then, for $S$ with $P_\bx(\bz^S) = 0$ the inputs to all output neurons $\ce_1$ are at most $-1$, and for $S$ with $P_\bx(\bz^S) = 1$ there exists a neuron in $\ce_1$ with input at least $2$. Recall that our definition of a neuron's input includes the addition of the bias term. The construction of the network $N_1$ is given in \lemref{lem:network N1 second layer}. Note that the network $N_1$ depends on $\bx$. Let $N'_1: \reals^{kn} \to \reals$ be a depth-$2$ neural network with no activation function in the output neuron, obtained from $N_1$ by summing the outputs from all neurons $\ce_1$.

Let $N_2$ be a depth-$1$ neural network (i.e., one layer, with activations in the output neurons) with input dimension $kn$, at most $2n$ output neurons, and parameter magnitudes bounded by $n^3$ (for a sufficiently large $n$), which satisfies the following. We denote the set of output neurons of $N_2$ by $\ce_2$. Let $\bz' \in \{0,1\}^{kn}$ be an input to $N_2$ (note that it has components only in $\{0,1\}$) . If $\bz'$ is an encoding of a hyperedge then the inputs to all output neurons $\ce_2$ are at most $-1$, and otherwise there exists a neuron in $\ce_2$ with input at least $2$. The construction of the network $N_2$ is given in \lemref{lem:network N2 second layer}. 
%Note that the network $N_2$ is independent of $\bx$. 
Let $N'_2: \reals^{kn} \to \reals$ be a depth-$2$ neural network with no activation function in the output neuron, obtained from $N_2$ by summing the outputs from all neurons $\ce_2$.

Let $N': \reals^{kn} \to \reals$ be a depth-$2$ network obtained from $N'_1,N'_2$ as follows. For $\bz' \in \reals^{kn}$ we have $N'(\bz') = \left[ 1 - N'_1(\bz') - N'_2(\bz')  \right]_+$. The network $N'$ has at most $n^2$ neurons, and parameter magnitudes bounded by $n^3$ (all bounds are for a sufficiently large $n$).
Finally, let $\tilde{N}:\reals^{n^2} \rightarrow \reals$ be a depth-$2$ neural network such that $\tilde{N}(\tilde{\bz}) = N'\left(\tilde{\bz}_{[kn]}\right)$.

\subsection{Defining the noise magnitudes $\tau,\omega$ and analyzing the perturbed network under perturbed inputs}

In order to use the algorithm $\cl$ w.r.t. some neural network with parameters $\btheta$ and a certain input distribution, we need to implement an examples oracle, such that the examples are drawn from a smoothed input distribution, and labeled according to a neural network with parameters $\btheta+\bxi$, where $\bxi$ is a random perturbation. Specifically, we use $\cl$ with an examples oracle where the input distribution $\hat{\cd}$ is obtained from $\cd$ by smoothing, and the labels correspond to a network $\hat{N}:\reals^{n^2} \to \reals$ obtained from $\tilde{N}$ (w.r.t. an appropriate $\bx \in \{0,1\}^n$ in the construction of $N_1$) by adding a small perturbation to the parameters. The smoothing magnitudes $\omega,\tau$ of the inputs and the network's parameters (respectively) are  
%we add i.i.d. noise to each parameter in $\tilde{N}$, where the noise is distributed according to $\cn(0,\tau^2)$, and $\tau=1/\poly(n)$ is small enough 
such that the following hold. 

We first choose the parameter $\tau = 1/\poly(n)$ as follows. 
Let $f_\btheta:\reals^{n^2} \to \reals$ be any depth-$2$ neural network parameterized by $\btheta \in \reals^r$ for some $r>0$ with at most $n^2$ neurons, and parameter magnitudes bounded by $n^3$ (note that $r$ is polynomial in $n$). 
Then, $\tau$ is such that with probability at least $1-\frac{1}{n}$ over $\bxi \sim \cn(\zero, \tau^2 I_r)$, we have $| \xi_i | \leq \frac{1}{10}$ for all $i \in [r]$, and the network $f_{\btheta + \bxi}$ is such that for every input $\tilde{\bz}  \in \reals^{n^2}$ with $\norm{\tilde{\bz}} \leq n$ and every neuron we have: Let $a,b$ be the inputs to the neuron in the computations $f_\btheta(\tilde{\bz})$ and $f_{\btheta + \bxi}(\tilde{\bz})$ (respectively), then $|a-b| \leq \frac14$.  Thus, $\tau$ is sufficiently small, such that w.h.p. adding i.i.d. noise $\cn(0,\tau^2)$ to each parameter does not change the inputs to the neurons by more than $\frac14$. Note that such an inverse-polynomial $\tau$ exists, since when the network size, parameter magnitudes, and input size are bounded by some $\poly(n)$, then the input to each neuron in $f_\btheta(\tilde{\bz})$ is $\poly(n)$-Lipschitz as a function of $\btheta$, and thus it suffices to choose $\tau$ that implies with probability at least $1 - \frac{1}{n}$ that $\norm{\bxi} \leq \frac{1}{q(n)}$ for a sufficiently large polynomial $q(n)$ (see \lemref{lem:tau exists} for details). 

Next, we choose the parameter $\omega =1/\poly(n)$ as follows. Let $f_\btheta:\reals^{n^2} \to \reals$ be any depth-$2$ neural network parameterized by $\btheta$ with at most $n^2$ neurons, and parameter magnitudes bounded by $n^3 + \frac{1}{10}$. Then, $\omega$ is such that for every $\bz \in \{0,1\}^{n^2}$, with probability at least $1-\exp(-n/2)$ over $\bzeta \sim \cn(\zero, \omega^2 I_{n^2})$ 
%we have $\norm{\bzeta} \leq 1$, and 
the following holds for every neuron in the $f_\btheta$: Let $a,b$ be the inputs to the neuron in the computations $f_\btheta(\bz)$ and $f_\btheta(\bz + \bzeta)$ (respectively), then $|a-b| \leq \frac14$.  Thus, $\omega$ is sufficiently small, such that w.h.p. adding noise $\cn(\zero,\omega^2 I_{n^2})$ to the input $\bz$ does not change the inputs to the neurons by more than $\frac14$. Note that such an inverse-polynomial $\omega$ exists, since when the network size and parameter magnitudes are bounded by some $\poly(n)$, then the input to each neuron in $f_\btheta(\bz)$ is $\poly(n)$-Lipschitz as a function of $\bz$, and thus it suffices to choose $\omega$ that implies with probability at least $1 - \exp(-n/2)$ that $\norm{\bzeta} \leq \frac{1}{q(n)}$ for a sufficiently large polynomial $q(n)$ (see \lemref{lem:tau exists} for details).

Let $\tilde{\btheta} \in \reals^p$ be the parameters of the network $\tilde{N}$. Recall that the parameters vector $\tilde{\btheta}$ is the concatenation of all weight matrices and bias terms. Let $\hat{\btheta} \in \reals^p$ be the parameters of $\hat{N}$, namely, $\hat{\btheta} = \tilde{\btheta} + \bxi$ where $\bxi \sim \cn(\zero, \tau^2 I_p)$. By our choice of $\tau$ and the construction of the networks $N_1,N_2$, with probability at least $1-\frac{1}{n}$ over $\bxi$, for every $\bz \in \{0,1\}^{n^2}$ the following holds: Let $\bzeta \sim \cn(\zero,\omega^2 I_{n^2})$ and let $\hat{\bz} = \bz + \bzeta$. Then with probability at least $1 - \exp(-n/2)$ over $\bzeta$ the differences between inputs to all neurons in the computations $\hat{N}(\hat{\bz})$ and $\tilde{N}(\bz)$ are at most $\frac{1}{2}$. Indeed, w.h.p. for all $\bz \in \{0,1\}^{n^2}$ the computations $\tilde{N}(\bz)$ and $\hat{N}(\bz)$ are roughly similar (up to change of  $1/4$ in the input to each neuron), and w.h.p. the computations $\hat{N}(\bz)$ and $\hat{N}(\hat{\bz})$ are roughly similar (up to change of $1/4$ in the input to each neuron). Thus, with probability at least $1-\frac{1}{n}$ over $\bxi$, the network $\hat{N}$ is such that for every $\bz \in \{0,1\}^{n^2}$, we have with probability at least $1 - \exp(-n/2)$ over $\bzeta$ that the computation $\hat{N}(\hat{\bz})$ satisfies the following properties, where $\bz' := \bz_{[kn]}$: 
\begin{enumerate}[label=(Q\arabic*)]
	\item If $\bz' = \bz^S$ for some hyperedge $S$, then the inputs to $\ce_1$ satisfy: \label{prop:first 2}
	\begin{itemize}
		\item If $P_\bx(\bz^S) = 0$ the inputs to all neurons in $\ce_1$ are at most $-\frac12$. 
		\item If $P_\bx(\bz^S) = 1$ there exists a neuron in $\ce_1$ with input at least $\frac32$.
	\end{itemize}
	\item  The inputs to $\ce_2$ satisfy:\label{prop:last 2}
	\begin{itemize}
		\item If $\bz'$ is an encoding of a hyperedge then the inputs to all neurons $\ce_2$ are at most $-\frac12$.
		\item Otherwise, there exists a neuron in $\ce_2$ with input at least $\frac32$.
	\end{itemize}
\end{enumerate}

\subsection{Stating the algorithm $\ca$}

Given a sequence $(S_1,y_1),\ldots,(S_{n^s},y_{n^s})$, where $S_1,\ldots,S_{n^s}$ are i.i.d. random hyperedges, the algorithm $\ca$ needs to distinguish whether $\by = (y_1,\ldots,y_{n^s})$ is random or that $\by = (P(\bx_{S_1}),\ldots,P(\bx_{S_{n^s}})) = (P_\bx(\bz^{S_1}),\ldots,P_\bx(\bz^{S_{n^s}}))$ for a random $\bx \in \{0,1\}^n$.
Let $\cs = ((\bz^{S_1},y_1),\ldots,(\bz^{S_{n^s}},y_{n^s}))$.

We use the efficient algorithm $\cl$ in order to obtain distinguishing advantage greater than $\frac{1}{3}$ as follows.
Let $\bxi$ be a random perturbation, and let $\hat{N}$ be the perturbed network as defined above, w.r.t. the unknown $\bx \in \{0,1\}^n$. Note that given a perturbation $\bxi$, only the weights in the second layer of the subnetwork $N_1$ in $\hat{N}$ are unknown, since all other parameters do not depend on $\bx$.
%, and we know the random perturbation $\bxi$.
The algorithm $\ca$ runs $\cl$ with the following examples oracle.
In the $i$-th call, the oracle first draws $\bz' \in \{0,1\}^{kn}$ such that each component is drawn i.i.d. from a Bernoulli distribution which takes the value $0$ with probability $\frac{1}{n}$. If $\bz'$ is an encoding of a hyperedge then the oracle replaces $\bz'$ with $\bz^{S_i}$. Let $\bz \in \{0,1\}^{n^2}$ be such that $\bz_{[kn]}=\bz'$, and the other $n^2-kn$ components of $\bz$ are drawn i.i.d. from a Bernoulli distribution which takes the value $0$ with probability $\frac{1}{n}$. 
Note that the vector $\bz$ has the distribution $\cd$, since replacing an encoding of a random hyperedge by an encoding of another random hyperedge does not change the distribution of $\bz'$. Let $\hat{\bz} = \bz + \bzeta$, where $\bzeta \sim \cn(\zero, \omega^2 I_{n^2})$. Note that $\hat{\bz}$ has the distribution $\hat{\cd}$.
Let $\hat{b} \in \reals$ be the bias term of the output neuron of $\hat{N}$.
The oracle returns $(\hat{\bz},\hat{y})$, where the labels $\hat{y}$ are chosen as follows:
\begin{itemize}
	\item If $\bz'$ is not an encoding of a hyperedge, then $\hat{y}=0$.
	\item If $\bz'$ is an encoding of a hyperedge:
    	\begin{itemize}
    		\item If $y_i=0$ we set $\hat{y} = \hat{b}$.
		\item If $y_i=1$ we set $\hat{y} = 0$.
    	\end{itemize}
\end{itemize}

Let $h$ be the hypothesis returned by $\cl$.
Recall that $\cl$ uses at most $m(n)$ examples, and hence $\cs$ contains at least $n^3$ examples that $\cl$ cannot view. We denote the indices of these examples by $I = \{m(n)+1,\ldots,m(n)+n^3\}$, and the examples by $\cs_I = \{(\bz^{S_i},y_i)\}_{i \in I}$. By $n^3$ additional calls to the oracle, the algorithm $\ca$ obtains the examples $\hat{\cs}_I = \{(\hat{\bz}_i,\hat{y}_i)\}_{i \in I}$ that correspond to $\cs_I$.
Let $h'$ be a hypothesis such that for all $\tilde{\bz} \in \reals^{n^2}$ we have $h'(\tilde{\bz}) = \max\{0,\min\{\hat{b},h(\tilde{\bz})\}\}$, thus, for $\hat{b} \geq 0$ the hypothesis $h'$ is obtained from $h$ by clipping the output to the interval $[0,\hat{b}]$.
Let $\ell_{I}(h')=\frac{1}{|I|}\sum_{i \in I}(h'(\hat{\bz}_i)-\hat{y}_i)^2$. 
Now, if $\ell_I(h') \leq \frac{2}{n}$, then $\ca$ returns $1$, and otherwise it returns $0$.
We remark that the decision of our algorithm is based on $h'$ (rather than $h$) since we need the outputs to be bounded, in order to allow using Hoeffding's inequality in our analysis, which we discuss in the next subsection.

\subsection{Analyzing the algorithm $\ca$}

Note that the algorithm $\ca$ runs in $\poly(n)$ time.
We now show that if $\cs$ is pseudorandom then $\ca$ returns $1$ with probability greater than $\frac{2}{3}$, and if $\cs$ is random then $\ca$ returns $1$ with probability less than $\frac{1}{3}$. To that end, we use similar arguments to the proof of  \thmref{thm:hard smoothed}.

In \lemref{lem:realizable2}, we show that if $\cs$ is pseudorandom then with probability at least $\frac{39}{40}$ (over $\bxi \sim \cn(\zero, \tau^2 I_p)$ and $\bzeta_i \sim \cn(\zero,\omega^2 I_{n^2})$ for all $i \in [m(n)]$) the examples $(\hat{\bz}_1, \hat{y}_1),\ldots,(\hat{\bz}_{m(n)},\hat{y}_{m(n)})$ returned by the oracle are realized by $\hat{N}$. 
%(where the subnetwork $N_1$ is defined w.r.t. the random $\bx$ that corresponds to the PRG).
Recall that the algorithm $\cl$ is such that with probability at least $\frac{3}{4}$ (over $\bxi \sim \cn(\zero, \tau^2 I_p)$, the i.i.d. inputs $\hat{\bz}_i \sim \hat{\cd}$, and possibly its internal randomness), given a size-$m(n)$ dataset labeled by $\hat{N}$, it returns a hypothesis $h$ such that $\E_{\hat{\bz} \sim \hat{\cd}} \left[(h(\hat{\bz})-\hat{N}(\hat{\bz}))^2 \right] \leq \frac{1}{n}$.
Hence, with probability at least $\frac{3}{4} - \frac{1}{40}$ the algorithm $\cl$ returns such a good hypothesis $h$, given $m(n)$ examples labeled by our examples oracle. 
Indeed, note that $\cl$ can return a bad hypothesis only if the random choices are either bad for $\cl$ (when used with realizable examples) or bad for the realizability of the examples returned by our oracle.
By the definition of $h'$ and the construction of $\hat{N}$, if $h$ has small error then $h'$ also has small error, namely, 
\[
	 \E_{\hat{\bz} \sim \hat{\cd}} \left[(h'(\hat{\bz})-\hat{N}(\hat{\bz}))^2 \right] 
	 \leq \E_{\tilde{\bz} \sim \hat{\cd}} \left[(h(\hat{\bz})-\hat{N}(\hat{\bz}))^2 \right] 
	 \leq \frac{1}{n}~.
\]

Let $\hat{\ell}_{I}(h')=\frac{1}{|I|}\sum_{i \in I}(h'(\hat{\bz}_i)-\hat{N}(\hat{\bz}_i))^2$.
Recall that by our choice of $\tau$ we have $\Pr[\hat{b} > \frac{11}{10}] \leq \frac{1}{n}$.
Since, $(h'(\hat{\bz})-\hat{N}(\hat{\bz}))^2 \in [0,\hat{b}^2]$ for all $\hat{\bz} \in \reals^{n^2}$,
by Hoeffding's inequality, we have for a sufficiently large $n$ that
\begin{align*}
	\Pr\left[\left|\hat{\ell}_{I}(h') -  \E_{\hat{\cs}_I}\hat{\ell}_{I}(h')\right| \geq \frac{1}{n}\right]
	&= \Pr\left[\left|\hat{\ell}_{I}(h') -  \E_{\hat{\cs}_I}\hat{\ell}_{I}(h')\middle| \geq \frac{1}{n} \right| \hat{b} \leq \frac{11}{10}\right] \cdot \Pr\left[ \hat{b} \leq \frac{11}{10} \right]
	\\
	&\;\;\;\; + \Pr\left[\left|\hat{\ell}_{I}(h') -  \E_{\hat{\cs}_I}\hat{\ell}_{I}(h')\middle| \geq \frac{1}{n} \right| \hat{b} > \frac{11}{10}\right] \cdot \Pr\left[ \hat{b} > \frac{11}{10} \right]
	\\
	&\leq 2 \exp\left( -\frac{2n^3}{n^2 (11/10)^4} \right) \cdot 1 + 1 \cdot \frac{1}{n}
	\\
	&\leq \frac{1}{40}~.
\end{align*}
Moreover, by \lemref{lem:realizable2},
\[
	\Pr \left[  \ell_I(h') \neq \hat{\ell}_I(h') \right]
	\leq \Pr \left[ \exists i \in I \text{ s.t. } \hat{y}_i \neq \hat{N}(\hat{\bz}_i) \right]
	\leq \frac{1}{40}~. 
\]

Overall, by the union bound we have with probability at least $1-\left( \frac{1}{4} + \frac{1}{40} + \frac{1}{40} + \frac{1}{40}\right)  > \frac{2}{3}$ for sufficiently large $n$ that:
\begin{itemize}
	\item $\E_{\hat{\cs}_I}\hat{\ell}_{I}(h') = \E_{\hat{\bz} \sim \hat{\cd}} \left[(h'(\hat{\bz})-\hat{N}(\hat{\bz}))^2 \right] \leq \frac{1}{n}$.
	\item $\left|\hat{\ell}_{I}(h') -  \E_{\hat{\cs}_I}\hat{\ell}_{I}(h')\right| \leq \frac{1}{n}$.
	\item $\ell_I(h') - \hat{\ell}_I(h') = 0$.
\end{itemize}
Combining the above, we get that if $\cs$ is pseudorandom, then with probability greater than $\frac{2}{3}$ we have
\[
	\ell_I(h')
	= \left( \ell_I(h') - \hat{\ell}_I(h') \right) + \left( \hat{\ell}_I(h') - \E_{\hat{\cs}_I}\hat{\ell}_{I}(h') \right) +\E_{\hat{\cs}_I}\hat{\ell}_{I}(h')  
	\leq 0 + \frac{1}{n} + \frac{1}{n} 
	= \frac{2}{n}~.
\]

We now consider the case where $\cs$ is random.
For an example $\hat{\bz}_i = \bz_i + \bzeta_i$ returned by the oracle, we denote $\bz'_i = (\bz_i)_{[kn]} \in \{0,1\}^{kn}$. Thus, $\bz'_i$ is the input that the oracle used before adding the $n^2-kn$ additional components and adding noise $\bzeta_i$.
Let $\cz' \subseteq \{0,1\}^{kn}$ be such that $\bz' \in \cz'$ iff $\bz'=\bz^{S}$ for some hyperedge $S$.
If $\cs$ is random, then by the definition of our examples oracle, for every $i \in [m(n) + n^3]$ such that $\bz'_i \in \cz'$, we have $\hat{y}_i=\hat{b}$ with probability $\frac{1}{2}$ and $\hat{y}_i=0$ otherwise. Also, by the definition of the oracle, $\hat{y}_i$ is independent of $S_i$, independent of the $n^2-kn$ additional components that where added, and independent of the noise $\bzeta_i \sim \cn(\zero, \omega^2 I_{n^2})$ that corresponds to $\hat{\bz}_i$.

If $\hat{b} \geq \frac{9}{10}$ then for a sufficiently large $n$ the hypothesis $h'$ satisfies for each random example $(\hat{\bz}_i,\hat{y}_i) \in \hat{\cs}_I$ the following:
\begin{align*} 
	\Pr_{(\hat{\bz}_i,\hat{y}_i)}&\left[(h'(\hat{\bz}_i)-\hat{y}_i)^2 \geq \frac{1}{5}\right]
	\\
	&\geq \Pr_{(\hat{\bz}_i,\hat{y}_i)} \left[\left.(h'(\hat{\bz}_i)-\hat{y}_i)^2 \geq \frac{1}{5} \; \right| \; \bz'_i \in \cz' \right] \cdot \Pr \left[\bz'_i \in \cz' \right]
	\\
	&\geq  \Pr_{(\hat{\bz}_i,\hat{y}_i)} \left[\left.(h'(\hat{\bz}_i)-\hat{y}_i)^2 \geq \left(\frac{\hat{b}}{2}\right)^2 \; \right| \; \bz'_i \in \cz' \right] \cdot \Pr \left[\bz'_i \in \cz' \right] 
	\\
	&\geq \frac{1}{2}  \cdot \Pr \left[\bz'_i \in \cz'\right]~.
\end{align*}
In \lemref{lem:prob z good discrete}, we show that for a sufficiently large $n$ we have $\Pr \left[\bz'_i \in \cz' \right] \geq \frac{1}{\log(n)}$.
Hence,
\begin{align*}
	\Pr_{(\hat{\bz}_i,\hat{y}_i)}  \left[(h'(\hat{\bz}_i)-\hat{y}_i)^2 \geq \frac{1}{5}\right]
	\geq  \frac{1}{2}  \cdot \frac{1}{\log(n)}
	\geq \frac{1}{2\log(n)}~.
\end{align*}
Thus, if $\hat{b} \geq \frac{9}{10}$ then we have
\[
	\E_{\hat{\cs}_I}\left[ \ell_I(h') \right] \geq \frac{1}{5} \cdot \frac{1}{2\log(n)} = \frac{1}{10\log(n)}~.
\]
Therefore, for large $n$ we have
\[
	\Pr\left[ \E_{\hat{\cs}_I}\left[ \ell_I(h') \right] \geq \frac{1}{10\log(n)} \right] \geq 1-\frac{1}{n} \geq \frac{7}{8}~.
\]

Since, $(h'(\hat{\bz})-\hat{y})^2 \in [0,\hat{b}^2]$ for all $\hat{\bz},\hat{y}$ returned by the examples oracle, and the examples $\hat{\bz}_i$ for $i \in I$ are i.i.d., then by Hoeffding's inequality, we have for a sufficiently large $n$ that
\begin{align*}
	\Pr\left[\left|\ell_{I}(h') -  \E_{\hat{\cs}_I}\ell_{I}(h')\right| \geq \frac{1}{n}\right]
	&= \Pr\left[\left|\ell_{I}(h') -  \E_{\hat{\cs}_I}\ell_{I}(h')\middle| \geq \frac{1}{n} \right| \hat{b} \leq \frac{11}{10}\right] \cdot \Pr\left[ \hat{b} \leq \frac{11}{10} \right]
	\\
	&\;\;\;\; + \Pr\left[\left|\ell_{I}(h') -  \E_{\hat{\cs}_I}\ell_{I}(h')\middle| \geq \frac{1}{n} \right| \hat{b} > \frac{11}{10}\right] \cdot \Pr\left[ \hat{b} > \frac{11}{10} \right]
	\\
	&\leq 2 \exp\left( -\frac{2n^3}{n^2 (11/10)^4} \right) \cdot 1 + 1 \cdot \frac{1}{n}
	\\
	&\leq \frac{1}{8}~.
\end{align*}

Hence, for large enough $n$, with probability at least $1 - \frac{1}{8} - \frac{1}{8} = \frac{3}{4} > \frac{2}{3}$ we have both $ \E_{\hat{\cs}_I}\left[ \ell_I(h') \right] \geq \frac{1}{10\log(n)}$ and $\left|\ell_{I}(h') -  \E_{\hat{\cs}_I}\ell_{I}(h')\right| \leq \frac{1}{n}$, and thus
\[
	\ell_I(h') \geq \frac{1}{10\log(n)} - \frac{1}{n} > \frac{2}{n}~.
\]

Overall, if $\cs$ is pseudorandom then with probability greater than $\frac{2}{3}$ the algorithm $\ca$ returns $1$, and if $\cs$ is random then with probability greater than $\frac{2}{3}$ the algorithm $\ca$ returns $0$. Thus, the distinguishing advantage is greater than $\frac13$.
This concludes the proof of the theorem. 
It remains to prove the deffered lemma on the realizability of the examples returned by the examples oracle:
%In the next subsection we provide the missing lemmas.

%\subsection{Missing lemmas}

\begin{lemma} \label{lem:realizable2}
	If $\cs$ is pseudorandom then with probability at least $\frac{39}{40}$ over $\bxi \sim \cn(\zero, \tau^2 I_p)$ and $\bzeta_i \sim \cn(\zero,\omega^2 I_{n^2})$ for $i \in [m(n)+n^3]$, the examples $(\hat{\bz}_1, \hat{y}_1),\ldots,(\hat{\bz}_{m(n)+n^3},\hat{y}_{m(n)+n^3})$ returned by the oracle are realized by $\hat{N}$.
\end{lemma}
\begin{proof}
	By our choice of $\tau$ and $\omega$ and the construction of $N_1,N_2$, with probability at least $1-\frac{1}{n}$ over $\bxi \sim \cn(\zero, \tau^2 I_p)$, we have $|\xi_j| \leq \frac{1}{10}$ for all $j \in [p]$, and for every $\bz \in \{0,1\}^{n^2}$ the following holds: Let $\bzeta \sim \cn(\zero,\omega^2 I_{n^2})$ and let $\hat{\bz} = \bz + \bzeta$. Then with probability at least $1 - \exp(-n/2)$ over $\bzeta$ the inputs to the neurons $\ce_1,\ce_2$ in the computation $\hat{N}(\hat{\bz})$ satisfy Properties~\ref{prop:first 2} and~\ref{prop:last 2}. Hence, with probability at least $1 - \frac{1}{n} - (m(n)+n^3) \exp(-n/2) \geq 1 - \frac{2}{n}$ (for a sufficiently large $n$), $|\xi_j| \leq \frac{1}{10}$ for all $j \in [p]$, and Properties~\ref{prop:first 2} and~\ref{prop:last 2} hold for the computations $\hat{N}(\hat{\bz}_i)$ for all $i \in [m(n)+n^3]$.
	 It remains to show that if $|\xi_j| \leq \frac{1}{10}$ for all $j \in [p]$ and Properties~\ref{prop:first 2} and~\ref{prop:last 2} hold, then the examples $(\hat{\bz}_1, \hat{y}_1),\ldots,(\hat{\bz}_{m(n)+n^3},\hat{y}_{m(n)+n^3})$ are realized by $\hat{N}$.
	
	Let $i \in [m(n)+n^3]$. 
	%For brevity, we denote $\hat{\bz} = \hat{\bz}_i$, $\hat{y} = \hat{y}_i$, and $\bz' = \tilde{\bz}_{[kn]}$.
	We denote $\hat{\bz}_i = \bz_i + \bzeta_i$, namely, the $i$-th example returned by the oracle was obtained by adding noise $ \bzeta_i$ to $\bz_i \in \{0,1\}^{n^2}$. We also denote $\bz'_i = (\bz_i)_{[kn]} \in \{0,1\}^{kn}$. 
	Since $|\xi_j| \leq \frac{1}{10}$ for all $j \in [p]$, and all incoming weights to the output neuron in $\tilde{N}$ are $-1$, then in $\hat{N}$ all incoming weights to the output neuron are in $\left[ - \frac{11}{10}, -\frac{9}{10} \right]$, and the bias term in the output neuron, denoted by $\hat{b}$, is in $\left[ \frac{9}{10}, \frac{11}{10} \right]$. 
	Consider the following cases:
	\begin{itemize}
		\item If $\bz'_i$ is not an encoding of a hyperedge then $\hat{y}_i = 0$. Moreover, in the computation $\hat{N}(\hat{\bz}_i)$, there exists a neuron in $\ce_2$ with output at least $\frac{3}{2}$ (by Property~\ref{prop:last 2}) . Since all incoming weights to the output neuron in $\hat{N}$ are in $\left[ - \frac{11}{10}, -\frac{9}{10} \right]$, and $\hat{b} \in \left[ \frac{9}{10}, \frac{11}{10} \right]$, then the input to the output neuron (including the bias term) is at most $\frac{11}{10} - \frac{3}{2} \cdot \frac{9}{10} < 0$, and thus its output is $0$.
	
	\item If $\bz'$ is an encoding of a hyperedge $S$, then by the definition of the examples oracle we have $S=S_i$. Hence:
		\begin{itemize}
			
			\item	If $y_i = 0$ then the oracle sets $\hat{y}_i=\hat{b}$. Since $\cs$ is pseudorandom, we have $P_\bx(\bz^S) = P_\bx(\bz^{S_i}) = y_i = 0$. Hence, in the computation $\hat{N}(\hat{\bz}_i)$ the inputs to all neurons in $\ce_1,\ce_2$ are at most $-\frac{1}{2}$ (by Properties~\ref{prop:first 2} and~\ref{prop:last 2}), and thus their outputs are $0$. Therefore, $\hat{N}(\hat{\bz}_i) = \hat{b}$.
				
			\item If $y_i = 1$ then the oracle sets $\hat{y}_i=0$. Since $\cs$ is pseudorandom, we have $P_\bx(\bz^S) = P_\bx(\bz^{S_i}) = y_i = 1$. Hence, in the computation $\hat{N}(\hat{\bz}_i)$ there exists a neuron in $\ce_1$ with output at least $\frac{3}{2}$ (by Property~\ref{prop:first 2}). Since all incoming weights to the output neuron in $\hat{N}$ are in $\left[ - \frac{11}{10}, -\frac{9}{10} \right]$, and $\hat{b} \in \left[ \frac{9}{10}, \frac{11}{10} \right]$, then the input to output neuron (including the bias term) is at most $\frac{11}{10} - \frac{3}{2} \cdot \frac{9}{10} < 0$, and thus its output is $0$.
			
		\end{itemize}
	
	\end{itemize}
\end{proof}

\end{document}